\documentclass{article}

\usepackage{arxiv}

\usepackage[utf8]{inputenc} 
\usepackage[T1]{fontenc}    
\usepackage{hyperref}       
\usepackage{url}            
\usepackage{booktabs}       
\usepackage{amsfonts}       
\usepackage{nicefrac}       
\usepackage{microtype}      

\usepackage{supertabular} 
\usepackage{array} 
\usepackage{tabularx} 
\usepackage{caption} 
\usepackage{diagbox} 

\usepackage[toc,page]{appendix}

\usepackage{amsmath} 
\usepackage{amssymb} 
\usepackage{calligra} 
\usepackage{mathabx} 
\usepackage{bm} 
\usepackage{stmaryrd} 
\usepackage{pifont} 
\usepackage{algpseudocode,algorithm,algorithmicx}
\newcommand*\Let[2]{\State #1 $\gets$ #2}
\usepackage{bbold} 
\usepackage{faktor} 

\usepackage{graphicx} 
\usepackage{float} 
\usepackage[caption = false]{subfig} 
\usepackage{tikz} 
\usetikzlibrary[patterns] 
\usetikzlibrary{calc,decorations.pathreplacing} 
\usetikzlibrary{shadows.blur} 
\usetikzlibrary{shapes.symbols}
\usepackage{epstopdf} 
\graphicspath{{/home/debortoli/img/bg_input/}{/home/debortoli/img/dreamsandmagic/}{/home/debortoli/img/kwatra}{/home/debortoli/img/simoncelli/simoncelli_original/}{/home/debortoli/img/Wei_Levoy/}}
\makeatletter
\usepackage[percent]{overpic} 
\usetikzlibrary{shapes.arrows}
\newcommand\appendtographicspath[1]{%
  \g@addto@macro\Ginput@path{#1}%
}
\makeatother

\usepackage[utf8]{inputenc} 
\usepackage[T1]{fontenc} 
\usepackage[english]{babel}
\usepackage{color} 
\usepackage{calc} 

\let\oldquote\quote
\let\endoldquote\endquote
\renewenvironment{quote}[2][]
  {\if\relax\detokenize{#1}\relax
     \def\quoteauthor{#2}%
   \else
     \def\quoteauthor{#2~---~#1}%
   \fi
   \oldquote}
  {\par\nobreak\smallskip\hfill(\quoteauthor)%
   \endoldquote\addvspace{\bigskipamount}} 

\newcommand{\norm}[1]{\| #1 \|}

\usepackage{xparse}
\usepackage{xifthen}

\def\eg{\textit{e.g.}}
\def\ie{\textit{i.e.}}
\newcommand{\quotem}[1]{%
``#1''
  }

\newcommand{\without}[1]{\backslash \lbrace #1 \rbrace}

\newcommand{\R}{\mathbb{R}}

\newcommand{\eqsp}{\;}

\newcommand{\Z}{\mathbb{Z}}
\newcommand{\N}{\mathbb{N}}

\newcommand{\M}{\mathcal{M}}

\newcommand{\vareps}{\varepsilon}

\newcommand{\argmin}{\operatorname{argmin} \ }

\newcommand{\PSNR}{$\operatorname{PSNR}$}

\newcommand{\SVP}{$\operatorname{SVP}$}

\newcommand{\NP}{$\operatorname{NP}$}

\newcommand{\NFA}{$\operatorname{NFA}$}
\newcommand{\KLT}{$\operatorname{KLT}$}

\newcommand{\MLE}{$\operatorname{MLE}$}

\DeclareDocumentCommand \mkmatrix {m} {
  \left( \begin{matrix} #1 \end{matrix} \right)}

\DeclareDocumentCommand \expec {o m} {
  \IfNoValueTF{#1}
  {\mathbb{E}\left[ #2 \right]}
  {\mathbb{E}_{#1} \left[ #2 \right]}}

\DeclareDocumentCommand \prob {o m} {
  \IfNoValueTF{#1}
  {\mathbb{P}\left[ #2 \right]}
  {\mathbb{P}_{#1} \left[ #2 \right]}}

\DeclareDocumentCommand \cov {o m} {
  \IfNoValueTF{#1}
  {\operatorname{Cov}\left[ #2 \right]}
  {\operatorname{Cov}_{#1} \left[ #2 \right]}}

\DeclareDocumentCommand \var {o m} {
  \IfNoValueTF{#1}
  {\operatorname{Var}\left[ #2 \right]}
  {\operatorname{Var}_{#1} \left[ #2 \right]}}

\newcommand{\autosim}{\mathcal{AS}}

\newcommand{\micro}{\mathbb{P}_0}

\newcommand{\autoprob}{\mathsf{AP}}

\newcommand{\autonfa}{\mathsf{ANFA}}

\newcommand{\acontrario}{\textit{a~contrario}}

\DeclareDocumentCommand \al {o m} {
  \begin{aligned}[#1] #2 \end{aligned}
}
\DeclareDocumentCommand \accolade {o m} {
  \left\lbrace \al[#1]{#2} \right.
}

  \DeclareDocumentCommand \seq {o o m m} {
    \IfNoValueTF{#2}
    {\IfNoValueTF{#1}
  {\left( #3_{#4} \right)_{#4 \in \N}}
  {\left( #3_{#4} \right)_{#4 \in #1}}}
{\left( #3_{#4} \right)_{#4 \in \llbracket #1, #2 \rrbracket}}
}
\newcommand{\NFAmax}{$\operatorname{NFA}_{\text{max}}$}
\newcommand{\NFAmaxmath}{\operatorname{NFA}_{\text{max}}}

\newcommand{\summ}[2]{\underset{#1}{\overset{#2}{\sum}}}

\newcommand{\vertt}[1]{\vert #1 \vert}
\newcommand{\veclet}[1]{\bm{\mathrm{#1}}}
\newcommand{\ind}[1]{\mathbb{1}_{#1}}

\DeclareDocumentCommand \assump {o m} {
  \begin{assumption}[A\ref{assump:#1}]
    \label{assump:#1}
    #2
  \end{assumption}
}

\newcommand*{\figuretitle}[1]{%
  \textit{\textbf{#1.}}
}

\usepackage{todonotes}

\newtheorem{thm}{Theorem} 
\newtheorem{mydef}{Definition} 
\newtheorem{prop}{Proposition} 

\newtheorem{assumption}{Assumption}

\usepackage{pgfplots}
\usetikzlibrary{spy}
\usepackage{mathrsfs}
\usepackage{enumitem}

\newenvironment{proof}{\paragraph{Proof:}}{\hfill$\square$}

\title{Patch redundancy in images: a statistical testing framework and some applications}
\author{Valentin De Bortoli \\
  CMLA, ENS Cachan, CNRS, Université Paris-Saclay, 94235 Cachan, France \\
  \And
  Agnès Desolneux \\
  CMLA, ENS Cachan, CNRS, Université Paris-Saclay, 94235 Cachan, France \\
  \And
  Bruno Galerne \\
  Institut Denis Poisson, Universit\'{e} d'Orléans, Universit\'{e} de Tours, CNRS \\
  \And
  Arthur leclaire \\
    Univ. Bordeaux, IMB, Bordeaux INP, CNRS, UMR 5251, F-33400 Talence, France.\\}
\begin{document}
  \maketitle
\begin{abstract}
~In this work we introduce a statistical framework in order to analyze the spatial redundancy in natural images.
This notion of spatial redundancy must be defined locally and thus we give some examples of functions (auto-similarity and template similarity) which, given one or two images, computes a similarity measurement between patches. Two patches are said to be similar if the similarity measurement is small enough. To derive a criterion for taking a decision on the similarity between two patches we present an \acontrario \ model. Namely, two patches are said to be similar if the associated similarity measurement is unlikely to happen in a background model. Choosing Gaussian random fields as background models we derive non-asymptotic expressions for the probability distribution function of similarity measurements. We introduce a fast algorithm in order to assess redundancy in natural images and present applications in denoising, periodicity analysis and texture ranking.  
\end{abstract}


\section{Introduction}
\label{sec:intro}
In many image processing applications, using local information combined with the knowledge of long-range spatial arrangement is crucial. The spatial redundancy on sub-images called patches, encodes the small scale structure of the image as well as its large scale organization. More precisely, local information is encoded in the patch content and the large scale organization is contained in the redundancy of this information across the patches of the image.
For example, patch-based inpainting techniques, such as \cite{criminisi2004region, he2014image}, assign patches of a known region to patches of an unknown region. Namely, each patch position on the border of the unknown region is associated to an offset corresponding to the best patch according to the partial available information. In \cite{he2014image} the authors replace the search on the whole image 
by a search among the most redundant offsets in the known region. This allows the authors of \cite{he2014image} to retrieve long-range spatial structure in the unknown part of the image. Another famous application 
of spatial redundancy can be found in denoising, with the seminal work (Non-Local means) of Buades and coauthors \cite{buades2005non}, in which the authors propose to replace a noisy patch by the mean over all spatially redundant patches.

Last but not least, spatial redundancy is of crucial importance in exemplar-based texture synthesis. In this paper we define textures as images containing repeated patterns but also reflecting randomness in the arrangement of these patterns. Among textures, one important class is given by the microtextures in which no individual object can be clearly delimited. In the periodic case, a more precise definition will be given in Definition \ref{def:microtexture}.
These microtexture models can be described by Gaussian random fields \cite{van1991spot, galerne2011random, leclaire2015random, xia2014synthesizing}. 
  Parametric models using features such as wavelet transform coefficients \cite{portilla2000parametric}, scattering transform coefficients \cite{sifre2013rotation} or convolutional neural network outputs \cite{gatys2015texture} have been proposed in order to derive image models with more structure
  . On the other hand, non-parametric patch-based algorithms such as \cite{efros1999texture,efros2001image,kwatra2003graphcut, raad2015conditional, galerne2018texture} propose to use most similar patches in order to fill the new texture images, similarly to inpainting techniques.


All these techniques lift images in spaces with dimensions higher than the original image space
, and make use of the redundancy of the lifting to extract important structural information. There exist two main types of lifting: feature extraction or patch extraction. Feature extraction relies on the use of filters, linear or non-linear, which aim at selecting substantial local information. Among popular kernels are oriented and multiscale filters, which happened to be identified as early processing in mammal vision systems \cite{daugman1985uncertainty, hubel1959receptive}. These last years have seen the rise of neural networks in which the filter dictionary is no longer given as an input but learned through a data-driven optimization procedure \cite{simonyan2014vgg}. On the other hand, patch-based methods rely on the assumption that image processing tasks are simplified when conducted in the higher dimensional patch space. 

Every analysis performed in a lifted space, built via feature extraction or patch extraction, relies on the comparison of points in this space. 
In patch-based lifted spaces, we aim at finding dissimilarity functions such that two patches are visually close if the dissimilarity measurement between them is small. In this paper we focus on the square Euclidean distance but other choices could be considered 
  \cite{wang2003multiscale,wang2004image,debortoli2018gaussian,deledalle2012compare}.

This leads us to consider a statistical hypothesis testing framework to assess similarity (or dissimilarity) between patches. The null hypothesis is defined as the absence of local structural similarities in the image. Reciprocally the alternative hypothesis is defined as the presence of such similarities. There exists a wide variety of tractable models exhibiting no similarity at long-range, like Gaussian random fields  \cite{van1991spot, galerne2011random, leclaire2015random, xia2014synthesizing} or spatial Markov random fields \cite{cross1983markov}, whereas sampling and inference in very structured models rely on optimization procedures and may be computationally expensive, their distribution being the limit of some Markov chain \cite{zhu1998filters,lu2015learning} or some stochastic optimization procedure \cite{bruna2018multiscale}. This encourages us to consider an \acontrario \ approach, \ie \ we do not consider the alternative hypothesis and focus on rejecting the null hypothesis. This framework was successfully applied in many areas of image processing \cite{davy2018reducing, desolneux2000meaningful, desolneux2001edge, almansa2003vanishing, cao2004application} and aims at identifying structure events in images. This statistical model takes its roots in the fundamental work of the Gestalt theory \cite{desolneux2007gestalt}. One of its principle, the non-accidentalness principle \cite{lowe2012perceptual} or Helmholtz principle \cite{zhu1999embedding, desolneux2001edge}, states that no structure is perceived in a noise model. 
To be precise, in our case of interest, we want to assess that no spatial redundancy is perceived in microtexture models. This methodology allows us to only design a locally structured background model to define a null hypothesis. Combining \acontrario \ principles and patch-based measures, we propose an algorithm to identify auto-similarities in images.

We then turn to the implementation of such an algorithm and illustrate the diversity of its possible applications with three examples: denoising, lattice extraction, and periodicity ranking of textures. In our denoising application we propose a modification of the celebrated Non-Local means algorithm \cite{buades2005non} (NL-means) by inserting a threshold in the selection of similar patches. Using an \acontrario \ model we are able to give probabilistic control on the patch reconstruction.

We then focus on periodicity detection and, more precisely, lattice extraction. Periodicity in images was described as an important feature 
in early mathematical vision~\cite{haralick1973textural}. Most of the proposed methods to analyze periodicity rely on global measurements such as the modulus of the Fourier transform  \cite{matsuyama1983structural} or the autocorrelation \cite{lin1997extracting}. These global techniques are widely used in crystallography where lattice properties, such as the angle between basis vectors, are fundamental  \cite{mevenkamp2015unsupervised, sang2014revolving}. Since all of our measurements are local, we are able to identify periodic similarities even in images which are not periodic but present periodic parts, for instance if two crystal structures are present in a single crystallography image. We draw a link between the introduced notion of auto-similarity and the inertia measurement in co-occurence matrices \cite{haralick1973textural}. We then introduce our lattice proposal algorithm which combines a detection map, \ie \ the output of our redundancy detection algorithm, and graphical model techniques, as in \cite{park2009deformed}, in order to extract lattice basis vectors. 

Our last application concerns texture ranking. Since the definition of texture is broad and covers a wide range of images, it is a natural question to identify criteria in order to distinguish textures. In \cite{liu2004computational}, the authors use a classical measure for distinguishing textures:~regularity. In this work, we narrow this criterion and restrict ourselves to the study of periodicity in texture images. The proposed graphical model inference naturally gives a quantitative measurement for texture periodicity ranking. We give an example of ranking on 25 images of the Brodatz set.

Our paper is organized as follows. An \acontrario \ framework for local similarity detection is 
proposed in Section \ref{sec:a_contrario_framework}. In the \acontrario \ framework, a background model, corresponding to the null hypothesis, is required. The consequence of choosing Gaussian models as background models is investigated and a redundancy detection algorithm is proposed in Section \ref{sec:gauss-model-cons}. The rest of the paper is dedicated to some examples of application of the proposed framework. After reviewing one of the most popular method in image denoising we introduce a denoising algorithm in Section~\ref{sec:nl-means-contrario-1} and present our experimental results in Section \ref{sec:expe-results}. Local dissimilarity measurements can be used as periodicity detectors. The link between the locality of the introduced functions and the literature on periodicity detection problems is investigated in Section \ref{sec:existing_algorithms}. An algorithm for detecting lattices in images is given in Section \ref{sec:algorithm and properites} and numerical results are presented in Section \ref{sec:experimental-results}. In our last experiment in Section \ref{sec:texture-rank}, we introduce a criterion for measuring texture periodicity. We conclude our study and discuss future work in Section \ref{sec:conclusion}.
\section{An a contrario framework for auto-similarity}
\label{sec:similarity functions}

We first introduce a notion of dissimilarity between patches of an input image.
\begin{mydef}[Auto-similarity]
  Let $u$ be an image defined over a domain $\Omega = \llbracket 0,M-1 \rrbracket^2 \subset \Z^2$, with $M \in \N \backslash \{ 0\}$. Let $\omega \subset \Z^2$ be a patch domain. We introduce ${P_{\omega}(u)= (\dot{u}(\veclet{y}))_{\veclet{y} \in \omega}}$ the patch at position $\omega$ in the periodic extension of $u$ to $\Z^2$, denoted by~$\dot{u}$. We define the auto-similarity with patch domain $\omega$ and offset $\veclet{t}\in \Z^2$ by
  \begin{equation}
    \autosim(u,\veclet{t},\omega) = \norm{P_{\veclet{t}+\omega}(u) - P_{\omega}(u)}_2^2 \eqsp .
  \end{equation}
  \label{def:autosim}
\end{mydef}
The auto-similarity computes the distance between a patch of $u$ defined on a domain $\omega$ and the patch of $u$ defined by the domain $\omega$ shifted by the offset vector $\veclet{t}$.

  In what follows, we introduce an \acontrario \ framework on the auto-similarity
  . This framework will allow us to derive an algorithm for detecting spatial redundancy in natural images.
\label{sec:a_contrario_framework}
In this section we fix an image domain $\Omega \subset \Z^2$ and a patch domain $\omega \subset \Omega$. We recall that our final aim is to design a criterion that will answer the following question: are two given patches similar? This criterion will be given by the comparison between the value of a dissimilarity function and a threshold $a$. We will define the threshold $a$ so that few similarities are identified in the null hypothesis model, \ie \ similarity does not occur ``just by chance''. Thus we can reformulate the initial question: is the similarity output of a dissimilarity function between two patches small enough? Or, to be more precise, how can we set the threshold $a$ in order to obtain a criterion for assessing similarity between patches?

This formulation agrees with the \acontrario \ framework \cite{desolneux2007gestalt} which states that geometrical and/or perceptual structure in an image is meaningful if it is a rare event in a background model. This general principle is sometimes called the Helmholtz principle \cite{zhu1999embedding} or the non-accidentalness principle \cite{lowe2012perceptual}. Therefore, in order to control the number of similarities identified in the background model, we study the probability density function of the auto-similarity 
function with input random image $U$ over $\Omega$. We will denote by $\micro$ the probability distribution of $U$ over $\R^{\Omega}$, the images over $\Omega$. We will assume that $\micro$ is a microtexture model, see Definition~\ref{def:microtexture} below for a precise definition of such a model. 
We define the following significant event which encodes spatial redundancy: $\autosim(u,\veclet{t},\omega) \leq a(\veclet{t})$,
where $a$, the threshold function, is defined over the offsets ($\veclet{t} \in \Z^2$) but also depends on other parameters such as $\omega$ or $\micro$ 
. The dependency of $a$ with respect to $\veclet{t}$ cannot be omitted. For instance, even in a Gaussian white noise $W$, the probability distribution function of $\autosim(W, \veclet{t}, \omega)$ depends on $\veclet{t}$. 


The Number of False Alarms (\NFA ) is a crucial quantity in the \acontrario \ methodology.
A false alarm is defined as an occurrence of the significant event in the background model~$\micro$. We recall that in our model the significant event is patch redundancy. This test must be conducted for every possible configurations of the significant event, \ie \ in our case we test every possible offset $\veclet{t}$. The \NFA \ is then defined as the expectation of the number of false alarms over all possible configurations
. Bounding the \NFA \ ensures that the probability of identifying $k$ offsets with spatial redundancy is also bounded, see Proposition \ref{prop:a_contrario_bound}. In what follows we give the definition of the \NFA \ in the spatial redundancy context.

\begin{mydef}[\NFA]
  Let $U \sim \micro$, where $\micro$ is a background microtexture model.
  We define the auto-similarity probability map $\autoprob$ for any $\veclet{t} \in \Omega$, $\omega \subset \Omega$ and $a \in \R^{\Omega}$ by
        \begin{equation}\autoprob(\veclet{t},\omega, a) =  \prob[0]{ \autosim(U,\veclet{t},\omega) \leq a(\veclet{t})}  \label{eq:def_autoprob} \eqsp .\end{equation}
      We define the auto-similarity expected number of false alarms $\autonfa$ by 
  \begin{equation}
    \label{eq:NFA}
  \autonfa(\omega, a) = \sum_{\veclet{t} \in \Omega} \autoprob(\veclet{t}, \omega, a) \eqsp .
  \end{equation}
  \label{def:NFA}
\end{mydef}  

Note that $\autoprob(\veclet{t}, \omega, a)$ corresponds to the probability that $\omega + \veclet{t}$ is similar to $\omega$ in the background model $U$.
  For any $\veclet{t} \in \Omega$, the cumulative distribution function of the auto-similarity random variable $\autosim(U,\veclet{t},\omega)$ under $\micro$ evaluated at value $\alpha(\veclet{t})$ is given by $\autoprob(\veclet{t},\omega,\alpha(\veclet{t}))$. We denote by ${q \mapsto \autoprob^{-1}(\veclet{t},\omega,q)}$ the inverse cumulative distribution function, potentially defined by a generalized inverse ($ \autoprob^{-1}(\veclet{t},\omega,q) = \inf \{\alpha(\veclet{t}) \in \R, \ \autoprob(\veclet{t}, \omega, \alpha(\veclet{t})) \geq q \}$), of the auto-similarity random variable for a fixed offset $\veclet{t}$, with $q \in (0,1)$ a quantile. 
  We now have all the tools to control the number of detected offsets in the background model.

  \begin{mydef}[Detected offset]
    Let $u \in \R^{\Omega}$ be an image, $\omega \subset \Omega$ a patch domain, and $a \in \R^{\Omega}$. An offset $\veclet{t}$ is said to be detected with respect to $a$, 
    if $\autosim(u,\veclet{t}, \omega) \leq a(\veclet{t})$. 
    \label{def:detec_offset}
  \end{mydef}
Note that a detected offset in $U \sim \micro$ 
corresponds to a false alarm in the \acontrario \ model. 
In what follows we suppose that the cumulative distribution function of $\autosim(U,\veclet{t}, \omega)$ is invertible for every $\veclet{t} \in \Omega$. This ensures that for any $\veclet{t} \in \Omega$ and $q \in (0,1)$ we have
  \begin{equation}
    \label{eq:invertibility}
    \autoprob\left(\veclet{t}, \omega, \autoprob^{-1}\left(\veclet{t},\omega, q\right)\right) = q \eqsp .
  \end{equation}

\begin{prop}
  \label{prop:a_contrario_bound}
  Let $\operatorname{NFA}_{\text{max}} \geq 0$ and for all $\veclet{t} \in \Omega$ define $ a(\veclet{t}) = \autoprob^{-1}\left(\veclet{t}, \omega, \operatorname{NFA}_{\text{max}} / |\Omega|\right)$. We have that for any $n \in \N \without{0}$,
  \begin{equation*}
    \autonfa(\omega, a) = \operatorname{NFA}_{\text{max}} \quad \text{and} \quad \prob[0]{ \text{\quotem{at least $n$ offsets are detected in $U$}}} \leq \frac{\operatorname{NFA}_{\text{max}}}{n} \eqsp.
  \end{equation*}


\end{prop}
\begin{proof}
  Using \eqref{eq:NFA}, and $a(\veclet{t}) = \autoprob^{-1}\left(\veclet{t}, \omega, \operatorname{NFA}_{\text{max}} / |\Omega|\right)$, we get
  \[ \autonfa(\omega, a) = \summ{\veclet{t} \in \Omega}{}{\autoprob(\veclet{t},\omega, a)} = \summ{\veclet{t} \in \Omega}{}{\autoprob\left(\veclet{t}, \omega, \autoprob^{-1}\left(\veclet{t},\omega, \operatorname{NFA}_{\text{max}} / \vertt{\Omega}\right)\right)} = \operatorname{NFA}_{\text{max}} \eqsp ,
  \]
  where the last equality is obtained using \eqref{eq:invertibility}.
    Concerning the upper-bound, we have, using the Markov inequality and \eqref{eq:def_autoprob}, for any $n \in \N \without{0}$
\begin{align*}
    \prob[0]{ \text{\quotem{\small at least $n$ offsets are detected in $U$}}} &= \prob[0]{\sum_{\veclet{t} \in \Omega}{}{\mathbb{1}_{\autosim(U, \veclet{t}, \omega) \leq a(\veclet{t})}} \ge n} \\ &\leq \frac{\sum_{\veclet{t} \in \Omega}{}{\expec{\mathbb{1}_{\autosim(U, \veclet{t}, \omega) \leq a(\veclet{t})}}}}{n}  \leq \frac{\operatorname{NFA}_{\text{max}}}{n} \eqsp ,
\end{align*}
where $\mathbb{1}_{\autosim(U, \veclet{t}, \omega) \leq a(\veclet{t})} = 1$ if $\autosim(U, \veclet{t}, \omega) \leq a(\veclet{t})$ and $0$ otherwise.
\end{proof}

Thus, setting $a$ as in Proposition \ref{prop:a_contrario_bound}, we have that an offset $\veclet{t} \in \Omega$ is detected for an image~$u \in \R^{\Omega}$ if
\begin{equation}\autosim(u,\veclet{t},\omega) \leq \autoprob^{-1}\left(\veclet{t},\omega, \operatorname{NFA}_{\text{max}} / \vertt{\Omega}\right) \eqsp . \label{eq:icdf_ineq}\end{equation}
This \acontrario \ detection framework 
can then be simply rewritten as 1) computing the auto-similarity function with input image $u$, 2) thresholding the obtained dissimilarity map with the inverse cumulative distribution function of the computed dissimilarity function under $\micro$. 
The computed threshold depends on the offset and Proposition \ref{prop:a_contrario_bound} ensures probabilistic guarantees on the expected number of detections under $\micro$. 
Using the inverse property of the inverse cumulative distribution function and \eqref{eq:icdf_ineq}, we obtain that an offset is detected if and only if
\begin{equation}\prob[0]{\autosim(U,\veclet{t},\omega) \leq \autosim(u,\veclet{t},\omega)}=  \autoprob\left(\veclet{t}, \omega, \autosim(u,\veclet{t},\omega)\right) \leq \operatorname{NFA}_{\text{max}} /\vertt{\Omega} \eqsp . \label{eq:true_detec}\end{equation}
Therefore, the thresholding operation can be conducted either on $\autosim(u,\veclet{t}, \omega)$, see \eqref{eq:icdf_ineq}, or on $\autoprob\left(\veclet{t}, \omega, \autosim(u,\veclet{t},\omega)\right)$, see \eqref{eq:true_detec}.
This property will be used in Section \ref{sec:detection-algorithm} to define a similarity detection algorithm based on the evaluation of $\autosim(u,\veclet{t}, \omega)$.
\section{Gaussian model and detection algorithm}
\label{sec:gauss-model-cons}

\subsection{Choice of background model}
\label{sec:choice-backgr-model}

In this section we compute $\autoprob \left( \veclet{t}, \omega, \alpha \right)$, \ie \ the cumulative distribution function of the similarity function under the null hypothesis model, with a Gaussian background model. Indeed, if the background model is simply a Gaussian white noise the similarities identified by the \acontrario \ algorithm are the ones that are not likely to be present in the Gaussian white noise image model. 
More generally, we consider stationary Gaussian random fields defined in the following way: 
we introduce an image $f$ over $\R^{\Omega}$ which contains the microtexture information we want to discard in our \acontrario \ model. 
In what follows we give the definition of the microtexture model associated to $f$.
\begin{mydef}[Microtexture model]
  \label{def:microtexture}
  Let $f \in \R^{\Omega}$, we define the associated microtexture model $U$ by setting, $U = f * W$, where $*$ is the periodic convolution operator over $\Omega$ given by $v * w(\veclet{x}) = \sum_{\veclet{y} \in \Omega} \dot{v}(\veclet{y}) \dot{w}(\veclet{x} - \veclet{y})$ and $W$ is a white noise over $\Omega$, \ie \ $(W(\veclet{x}))_{\veclet{x} \in \Omega}$ are i.i.d. $\mathcal{N}(0,1)$ random variables.
\end{mydef}
  Given an image $u \in \R^{\Omega}$, a microtexture model can be derived considering
  \begin{equation}m_u = \sum_{\veclet{x} \in \Omega} u(x)/|\Omega|  \eqsp , \quad \text{and} \quad  U = |\Omega|^{-1/2} ( u  - m_u)* W \eqsp . \label{eq:gaussian_model}\end{equation}
Note that if $U$ is given by \eqref{eq:gaussian_model} we have for any $\veclet{x}, \veclet{y} \in \Omega$
\begin{equation}
  \expec{U(\veclet{x})} = 0\quad \text{and} \ \cov{U(\veclet{x}), U(\veclet{y})} = |\Omega|^{-1} \sum_{\veclet{z} \in \Omega}(\dot{u}(\veclet{z}) - m_u)(\dot{u}(\veclet{z - (y-x)}) - m_u) \eqsp .
\end{equation}
We refer to \cite{galerne2011random} for a mathematical study of this model.

\begin{figure}[h]
  \centering
  \subfloat[]{\includegraphics[width=.3\linewidth]{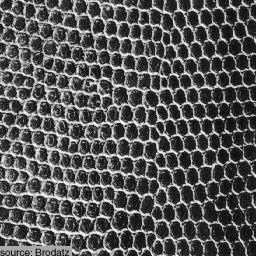}}\hfill
  \subfloat[]{\includegraphics[width=.3\linewidth]{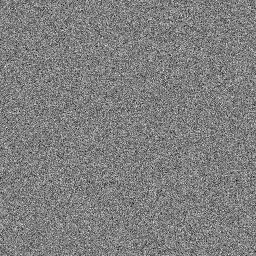}}\hfill
  \subfloat[]{\includegraphics[width=.3\linewidth]{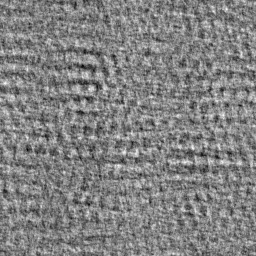}}\hfill  
  \caption{\figuretitle{Examples of microtexture models} In~(a) we present an original $256 \times 256$ image. In~(b) and~(c) we derive two microtexture models. In~(b) we present a Gaussian white noise and in~(c) the microtexture model given by \eqref{eq:gaussian_model}. Note that (c) shows more local structure than (b).}
  \label{fig:illus_dmap}
\end{figure}

\subsection{Detection algorithm}
\label{sec:detection-algorithm}
In this section, $\Omega$ is a finite square domain in $\Z^2$. We fix $\omega \subset \Omega$. We also define $f$, a 
function over $\Omega$. We consider the Gaussian random field $U = f * W$, where $W$ is a Gaussian white noise over $\Omega$. We denote by $\Gamma_f$ the autocorrelation of $f$, \ie \ $\Gamma_f = f * \check{f}$ where for any $\veclet{x} \in \Omega$, $\check{f}(\veclet{x}) = f(-\veclet{x})$. We introduce the offset correlation function~$\Delta_f$ defined for any $\veclet{t}, \veclet{x} \in \Omega$ by
\begin{equation}
  \label{eq:delta_fun}
  \Delta_f(\veclet{t}, \veclet{x}) = 2\Gamma_f(\veclet{x}) - \Gamma_f(\veclet{x+t}) - \Gamma_f(\veclet{x-t}) \eqsp . \end{equation}
The following proposition, proved in \cite{debortoli2018gaussian}, gives the explicit probability distribution function of the squared $\ell^2$ auto-similarity.
\begin{prop}[Squared $\ell^2$ auto-similarity function exact probability distribution function]
  Let $\Omega = \llbracket 0, M-1 \rrbracket^2$ with $M \in \N \backslash \{ 0 \}$, $\omega \subset \Omega$, $f \in \R^{\Omega}$ and $U = f* W$ where $W$ is a Gaussian white noise over $\Omega$.
  Then, for any $\veclet{t} \in \Omega$, $\autosim(U,\veclet{t},\omega)$ has the same distribution as $\sum_{k=0}^{|\omega| - 1}{\lambda_k(\veclet{t},\omega)Z_k}$, 
  with $Z_k$ independent chi-square random variables with parameter 1 and $\lambda_k(\veclet{t},\omega)$ the eigenvalues of the covariance matrix $C_{\veclet{t}}$ associated with function $\Delta_f(\veclet{t},\cdot)$ restricted to $\omega$, defined in \eqref{eq:delta_fun}, i.e \ for any $\veclet{x_1}, \veclet{x_2} \in \omega$, 
  $C_{\veclet{t}}(\veclet{x_1, x_2}) = \Delta_f(\veclet{t}, \veclet{x_1 - x_2})$.
  \label{prop:squared_exact}
\end{prop}
As a consequence if $f =\delta_0$, \ie \ $U$ is a Gaussian white noise, and $\{ \veclet{x} + \veclet{t}, \veclet{x} \in \omega\} \cap \omega = \emptyset$, \ie \ there is no overlapping between the patch domain $\omega$ and its shifted version, then $\autosim(U,\veclet{t},\omega)$ is a chi-square random variable with parameter $|\omega|$.

In order to compute the cumulative distribution function of a quadratic form of Gaussian random variables we must deal with two issues: 1) the computation of the eigenvalues $\lambda_k(\veclet{t}, \omega)$ might be time-consuming and efficient methods must be developed ; 2) the exact computation of the cumulative distribution function of a quadratic form of Gaussian random variables requires the use of heavy integrals, see \cite{imhof1961computing}. In \cite{debortoli2018gaussian} a projection method is introduced in order to easily compute approximated eigenvalues, with equality when $\omega = \Omega$. 
The so-called Wood F method (see \cite{wood1989f, bodenham2016comparison}) shows the best trade-off between accuracy and computational cost to approximate the cumulative distribution function of quadratic forms in Gaussian random variables with given weights. It is a moment method of order 3, fitting a Fisher-Snedecor distribution to the empirical one. Note that in \cite{liu2009chisquare} another moment method of order 3 is proposed. In what follows, we assume that we can compute the cumulative distribution function  of $\autosim(U,\veclet{t},\omega)$ and we refer to \cite{debortoli2018gaussian} for further details.

In Algorithm \ref{alg:auto-similaritydetection} we propose an \acontrario \ framework for spatial redundancy detection. We suppose that $u$ and $\omega$ are provided by the user. 
Using Proposition \ref{prop:a_contrario_bound} and \eqref{eq:true_detec} 
, we say that an offset is detected if  $\autoprob\left(\veclet{t}, \omega, \autosim(u,\veclet{t},\omega)\right) \leq \operatorname{NFA}_{\text{max}} /\vertt{\Omega}$. The value \NFAmax \ is supposed to be set by the user. 
The background model used in the auto-similarity detection is the one given in \eqref{eq:gaussian_model}.
Therefore, Proposition \ref{prop:squared_exact} and the discussion that follows can be used to compute an approximation of $\autoprob(\veclet{t}, \omega, \autosim(u, \veclet{t},\omega))$. In Figure \ref{fig:illus_dmap} we apply Algorithm \ref{alg:auto-similaritydetection} to a texture image.


\begin{algorithm}
  \caption{Auto-similarity detection
    \label{alg:auto-similaritydetection}}
  \begin{algorithmic}[1]
    \Function{autosim-detection}{$u$, $\omega$, \NFAmax}
    \For{$\veclet{t} \in \Omega$}
    \Let{val}{$\autosim(u,\veclet{t},\omega)$}
    \Let{$P_{map}(\veclet{t})$}{$\autoprob(\veclet{t}, \omega, \text{val})$} \Comment{$\autoprob(\veclet{t}, \omega, \text{val})$ approximation detailed in Section \ref{sec:detection-algorithm}}
    \Let{$D_{map}(\veclet{t})$}{$\mathbb{1}_{P_{map}(\veclet{t}) \leq \text{\NFAmax} / |\Omega |}$} 
    \EndFor    
      \State \Return{the images $P_{map}$, $D_{map}$}
    \EndFunction
  \end{algorithmic}
\end{algorithm}

\begin{figure}
  \centering
  \subfloat[]{\includegraphics[width=.19\linewidth]{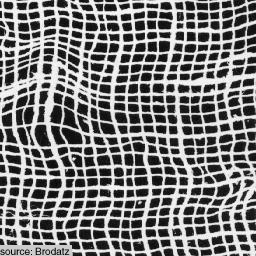}}\hfill
  \subfloat[]{\includegraphics[width=.19\linewidth]{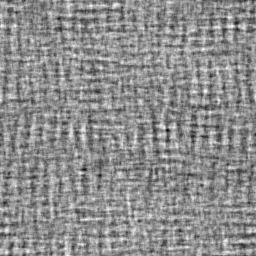}}\hfill
  \subfloat[]{\includegraphics[width=.19\linewidth]{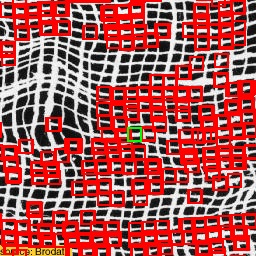}}\hfill  
  \subfloat[]{\includegraphics[width=.19\linewidth]{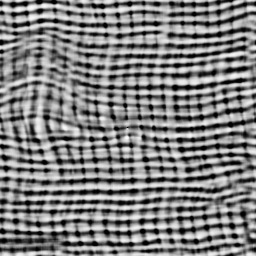}}\hfill
  \subfloat[]{\includegraphics[width=.19\linewidth]{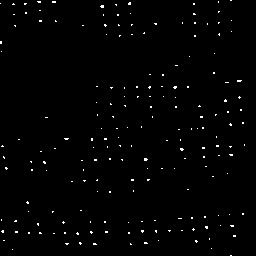}}\hfill
  \caption{\figuretitle{Outputs of Algorithm \ref{alg:auto-similaritydetection}} In~(a) we present an original $256 \times 256$ image. In~(b) we present the associated microtexture model given by \eqref{eq:gaussian_model}. In~(c) the green patch is the input patch, \ie \ $P_{\omega}(u)$. In this experiment \NFAmax \ is set to $1$. In~(d), respectively (e), we present the output $P_{map}$, respectively $D_{map}$, of Algorithm \ref{alg:auto-similaritydetection}. In~(c) we show in red the patches corresponding to the identified offsets in $P_{map}$.}
  \label{fig:illus_dmap}
\end{figure}

\section{Denoising}
\label{sec:denoising}
\subsection{NL-means and a contrario framework}
\label{sec:nl-means-contrario-1}
In this section we apply the \acontrario \ framework to the context of image denoising and propose a simple modification of the celebrated image denoising algorithm Non-Local Means (NL-means). This algorithm was introduced in the seminal paper of Buades et al. \cite{buades2005non} and was inspired by the work of Efros and Leung in texture synthesis \cite{efros1999texture}. It was also independently introduced in \cite{awate2006unsupervised}. This algorithm relies on the simple idea that denoising operations can be conducted in the lifted patch space. In this space the usual Euclidean distance acts as a good similarity detector and we can obtain a denoised patch by averaging all the patches with weights that depend on this Euclidean distance. Usually the weight function is set to have exponential decay, but it was suggested in \cite{goossens2008improved, salmon2010two, duval2010parameter} to use compactly supported weight functions in order to avoid the loss of isolated details. Since its introduction, many algorithms derived from NL-means have been proposed in order to embed the algorithm in general statistical frameworks \cite{duval2011bias, lebrun2013nonlocal} or to take into account the underlying geometry of the patch space \cite{houdard2017high}. Among the state-of-the-art denoising algorithms, see \cite{lebrun2012secrets} for a review, we consider Block-Matching and 3D Filtering (BM3D) \cite{dabov2007image} to compare our algorithm with.

There exist several works combining \acontrario \ models and denoising tasks. Coupier et al. in \cite{coupier2005image} propose to combine morphological filters and a testing hypothesis framework to remove impulse noise. In \cite{delon2013patch} Delon and Desolneux compare different statistical frameworks to perform denoising with Gaussian noise or impulse noise. The \acontrario \ model was also successfully used to deal with speckle noise \cite{fablet2005speckle} and quasi-periodic noise \cite{sur2015contrario}, and rely on the thresholding of wavelet or Fourier coefficients. In \cite{kervrann2008local}, Kervrann and Boulanger derive approximated probabilistic thresholds using $\chi_2$ probability distribution functions. In \cite{wue2013probabilistic} the authors propose a testing framework in order to estimate thresholds. The expressions they derive also relies on an approximation of the probability distribution of the squared Euclidean norm between two patches in Gaussian white noise. 

Following a standard extension procedure of the NL-means algorithm we consider a threshold version of it, see Algorithm \ref{alg:nlmeans_thres}. In what follows we fix a ``clean'', or original, image $u_0$ defined over $\Omega$, a finite rectangular domain of $\Z^2$, a noisy image $u = u_0 + \sigma w$, with $w$ a realization of a standard Gaussian random field $W$ and $\sigma >0$ the standard deviation of the noise. In all of our experiments we suppose that $\sigma$ is known. Note that there exist several algorithms to estimate $\sigma$ from real images, see \cite{ponomarenko2007noise} for instance. Our goal is to retrieve $u_0$ based on the information in $u$. We consider the lifted version of $u$ in a patch space. Let $\omega_0$ be a centered $8 \times 8$ patch domain. For a patch window $\omega = \veclet{x} + \omega_0$ the patch search window $T$ will be defined by
\begin{equation}T = \left\lbrace \veclet{t} \in \Z^2, \ \veclet{t} + \omega \subset \Omega, \ \| \veclet{t} \|_{\infty} \leq c \right\rbrace \eqsp , \label{eq:patch_search_window}
\end{equation}
with $c \in \N$. $|T|$ denotes the cardinality of $T$. There exists a large literature concerning the setting of $c$ and $\omega_0$, see \cite{duval2010parameter}. Note that the locality of the patch window was assessed to be a crucial feature of NL-means \cite{grewenig2011rotationally}. Suppose we have a collection of denoised patches $\hat{p}(u, \omega)$ for all patch domains  $\omega$, we obtain a pixel at position $\veclet{x}$ in the denoised image $\hat{u}$ using the following average, see \cite{buades2011non}, 
  \begin{equation}
    \hat{u}(\veclet{x}) =  \left| \lbrace \veclet{t} \in \Omega, \ \text{s.t} \ \veclet{x} \in \veclet{t} + \omega \subset \Omega \rbrace \right|^{-1}
  \sum_{\veclet{t} \in \Omega, \ \text{s.t} \ \veclet{x} \in \veclet{t} + \omega \subset \Omega} \hat{p}(u,\veclet{t} + \omega)(\veclet{x})   \eqsp . \label{eq:mean_denoising}
  \end{equation}
 We now introduce our modification of NL-means. We suppose that we are provided a threshold function $a$. The choice of such a function is discussed in Proposition \ref{prop:a_contrario_bound_nlmeans}.

\begin{algorithm}
  \caption{NL-means threshold
    \label{alg:nlmeans_thres}}
  \begin{algorithmic}[1]
    \Function{NL-means-threshold}{$u$, $\sigma$, $\omega_0$, $c$, $a$}
    \For{$\veclet{x} \in \Z^2, \ \veclet{x} + \omega_0 \subset \Omega$}
    \Let{$\omega$}{$\veclet{x} + \omega_0$}
    \Let{$T$}{defined by \eqref{eq:patch_search_window}}
    \Let{$N_{\omega}(u)$}{0}
    \Let{$\hat{p}(u, \omega)$}{0}
    \For{$\veclet{t} \in T$}
    \If{$\autosim(u,\veclet{t},\omega) \leq \sigma^2 a(\veclet{t})$} \Comment{always true for $\veclet{t} =0$}
    \Let{$\hat{p}(u, \omega)$}{$\frac{N_{\omega}(u)}{N_{\omega}(u) + 1} \hat{p}(u,\omega) + \frac{1}{N_{\omega}(u)+1} P_{\veclet{t} + \omega}(u)$ \Comment{$P_{\omega}(u)$ is given in Definition \ref{def:autosim}} 
    }
    \Let{$N_{\omega}(u)$}{$N_{\omega}(u) +1$}
    \EndIf
    \EndFor
    \EndFor
    \Let{$\hat{u}$}{defined by \eqref{eq:mean_denoising}}
      \State \Return{$\hat{p}(u, \cdot)$, $\hat{u}$}
    \EndFunction
  \end{algorithmic}
\end{algorithm}

Note here that the output denoised version of the patch $\hat{p}(u, \omega)$ verifies the following equation
\begin{equation*}
  \hat{p}(u, \omega) = \sum_{\veclet{t} \in T} \lambda_{\veclet{t}} P_{\veclet{t} + \omega}(u) \eqsp , \qquad \lambda_{\veclet{t}} = \frac{\ind{\autosim(u,\veclet{t},\omega) \leq a(\veclet{t})}}{\sum_{\veclet{s} \in T} \ind{\autosim(u,\veclet{s},\omega) \leq a(\veclet{s})}}  \eqsp .
\end{equation*}
In the original NL-means method, we have 
\begin{equation}\lambda_{\veclet{t}} = \frac{\exp\left( -\frac{\autosim(u, \veclet{t}, \omega)}{h^2}\right)}{\sum_{\veclet{t} \in T} \exp\left( -\frac{\autosim(u, \veclet{t}, \omega)}{h^2}\right)} \eqsp . \label{eq:original_nlmeans}\end{equation} Setting $h$ is not trivial and depends on many parameters (patch size, search window size, content of the original image). 
As in Algorithm \ref{alg:nlmeans_thres}, we denote $N_{\omega}(u) = \sum_{\veclet{t} \in T} \ind{\autosim(u, \veclet{t}, \omega) \leq a(\veclet{t}) } $. The following proposition, similar to Proposition \ref{prop:a_contrario_bound}, gives a method for setting $a$. We say that an offset $\veclet{t}$ is a false alarm in a Gaussian white noise if the associated patch is not used in the denoising algorithm. In Proposition \ref{prop:a_contrario_bound_nlmeans} we choose $a$ in order to control the number of false alarms with high probability.

\begin{prop}
  \label{prop:a_contrario_bound_nlmeans}
  Let $\operatorname{NFA}_{\text{max}} \in [0,|T|]$, $T$ given in \eqref{eq:patch_search_window} and let $a \in \R^{\Omega}$ be defined for any $\veclet{t} \in \Omega$ by \[a(\veclet{t}) = \autoprob^{-1}\left(\veclet{t}, \omega, 1 - \operatorname{NFA}_{\text{max}} / |T|\right) \eqsp ,\]with background model being a Gaussian white noise $W$, \ie \ $f= \delta_0$ in  Definition \ref{def:microtexture}. Let $T$ be defined in \eqref{eq:patch_search_window} and $N_{\omega}(W) \in \lbrace 0, \dots, T \rbrace$ the random number of selected patches used to denoise the patch $P_{\omega}(W)$, see Algorithm \ref{alg:nlmeans_thres}. Then for any $n \in \N \without{0}$ it holds that
  \[ \prob[0]{|T| - N_{\omega}(W) \geq  n} \leq \frac{\NFAmaxmath}{n}  \eqsp .\]
\end{prop}
\begin{proof}
Using the Markov inequality, we have \[
    \prob[0]{|T| - N_{\omega}(W) \geq n}
                                  \leq \frac{|T| - \sum_{\veclet{t} \in T} \expec{\ind{\autosim(W,\veclet{t}, \omega) \leq a(\veclet{t})}}}{n}
                                    \leq \frac{\NFAmaxmath}{n} \eqsp .
 \]
                                  \end{proof}

  In this case the null hypothesis $\micro$ is given by a standard Gaussian random field, which is a special case of the Gaussian random field models introduced in Section \ref{sec:gauss-model-cons}.
 In the next proposition, using the \acontrario \ framework, we obtain probabilistic guarantees on the distance between the reconstructed patch $\hat{p}(u, \omega)$ and the true patch $P_{\omega}(u_0)$. 
  \begin{prop}
    \label{prop:reconstruction}
    Let $U = u_0 + \sigma W$, where $W$ is a standard Gaussian white noise over $\Omega$, $u_0 \in \R^{\Omega}$ and $\sigma >0$.
    Let $\veclet{x} \in \Omega$ and $\omega = \veclet{x} + \omega_0$ be a fixed patch and let $\NFAmaxmath \in [0,|T|]$.
    We introduce the random set $\hat{T} = \lbrace \veclet{t} \in T, \ \autosim(U,\veclet{t}, \omega) \leq \sigma^2 a(\veclet{t}) \rbrace$ (the selected offsets) with  $a(\veclet{t}) = \autoprob^{-1}\left(\veclet{t}, \omega, 1 - \NFAmaxmath/|T|\right)$ as in Proposition \ref{prop:a_contrario_bound_nlmeans} and $T$ defined in \eqref{eq:patch_search_window}. Let $a_T = \max_{\veclet{t} \in T} a(\veclet{t})$.
     Then for any $a_W>0$, setting $\vareps_{W} = 1- \prob{\| P_{\omega}(W) \|_2^2 \leq a_W \ | \ \hat{T}}$, we have
    \begin{equation}
      \label{eq:upp_bound_nlmeans}
      \prob{\| \hat{p}(U, \omega) - P_{\omega}(u_0) \|_2 \leq \sigma (a_T^{1/2} + a_W^{1/2}) \ | \ \hat{T}} \geq 1 - \vareps_W\eqsp .
    \end{equation}
  \end{prop}
  \begin{proof}
    We have for any $\veclet{t} \in \hat{T}$
    \[\al{\| P_{\veclet{t} + \omega}(U) - P_{\omega} (u_0) \|_2 &\leq \| P_{\veclet{t} + \omega}(U) - P_{\omega}(U) + P_{\omega}(U) - P_{\omega}(u_0) \|_2 \\
        &\leq \| P_{\veclet{t} + \omega}(U) - P_{\omega}(U) \|_2 + \| P_{\omega}(U) - P_{\omega}(u_0) \|_2 \\
        &\leq \sigma a_T^{1/2} +\sigma  \| P_{\omega}(W) \|_2 \eqsp .} \]
    This gives the following event inclusion for any $\veclet{t} \in \hat{T}$,
    \begin{equation*}
       \left\lbrace \| P_{\omega}(W) \|_2 \leq a_W^{1/2}  \right\rbrace \subset \left\lbrace \| P_{\veclet{t} + \omega}(U) - P_{\omega} (u_0) \|_2 \leq \sigma ( a_T^{1/2} + a_W^{1/2} )  \right\rbrace \eqsp ,
   \end{equation*}
    We also have that by definition of $\vareps_W$
    \begin{align*}
      &\prob{\| \hat{p}(U, \omega) - P_{\omega}(u_0) \|_2 \leq \sigma (a_T^{1/2}+a_W^{1/2})  \ | \  \hat{T}}  \\ &\phantom{aaaaaaaaaaaaaaa}
                                                                                                              \geq \prob{\bigcap_{\veclet{t} \in \hat{T}}  \lbrace \| P_{\veclet{t} + \omega}(U) - P_{\omega}(u_0) \|_2^2  \leq \sigma^2 (a_T^{1/2}+a_W^{1/2})^2 \rbrace \ | \   \hat{T} } \\ &\phantom{aaaaaaaaaaaaaaa}
      \geq \prob{ \| P_{\omega}(W) \|_2^2 \leq a_W \ | \  \hat{T} } \geq 1 - \vareps_W \eqsp .
    \end{align*}
  \end{proof}

    In our applications we use Algorithm \ref{alg:nlmeans_thres} with $a(\veclet{t}) = \autoprob^{-1}\left(\veclet{t}, \omega, 1 - \NFAmaxmath/|T|\right)$. Therefore we need to compute $a(\veclet{t}) = \autoprob^{-1}\left(\veclet{t}, \omega, 1 - \NFAmaxmath/|T|\right)$ with a Gaussian white noise background model. We recall that in Section \ref{sec:detection-algorithm}, using Proposition \ref{prop:squared_exact}, we give a method to compute this quantity in general Gaussian settings. 
In the case of a Gaussian white noise, the next proposition shows that the eigenvalues can be computed without approximation.

    \begin{prop}
      \label{prop:eigenvalues}
      Let $\veclet{t} = (t_x, t_y) \in \Z^2 \without{0}$, $C_{\veclet{t}}$ as in Proposition \ref{prop:squared_exact} with $f = \delta_0$ and $\omega = \llbracket 0, p-1 \rrbracket^2$, with $p \in \N$. 
      We have, expressing $C_{\veclet{t}}$ in the basis corresponding to the raster scan order on the $x$-axis
      \begin{equation*}
        C_{\veclet{t}} = \mkmatrix{B_0 & B_1 & \dots & B_{p-1} \\
          B_1^{\top} & B_0 & \ddots & \vdots \\
          \vdots & \ddots & B_0 & B_1 \\
          B_{p-1}^{\top} & \dots & B_1^{\top} & B_0} + 2 \mathrm{Id} \eqsp ,  \quad \begin{cases} B_{\ell} = D_{|t_y|} \in \M_{p}(\R) & \text{if } \ell = |t_x| \\
        B_{\ell} = 0 & \text{otherwise} \end{cases}
    \end{equation*}
    where $D_j$ is a zero matrix with ones on the $j$-th diagonal
    . The eigenvalues of $C_{\veclet{t}}$ are given by $\lambda_{m,k} = 4 \sin^2\left( \frac{k\pi}{2m} \right)$ with multiplicity $r_{m,k}$ where $ m \in \llbracket 2, q + 1 \rrbracket$, $k \in \llbracket 1, m - 1 \rrbracket$ and $q = \lceil \frac{p}{|t_x| \vee |t_y|} \rceil $. For any $m \in \llbracket 2, q+1 \rrbracket$, $k \in \llbracket 1, m-1 \rrbracket$ it holds 
    \begin{enumerate}[label=(\alph*)]
        \item for any $k' \in \llbracket 1, m-1 \rrbracket$, $r_{m,k} = r_{m,k'} \eqsp ;$
        \item $r_{m,k} = 2 |t_x| |t_y| \ \text{if} \ 2\leq m < q \eqsp ;$
        \item $r_{m,k} = r_x r_y \ \text{if} \ m = q +1  \eqsp ;$
        \item $ \sum_{m=2}^{q+1} \sum_{k=1}^{m-1} r_{m,k} = p^2 \eqsp ,$
      \end{enumerate}
      with $r_x = \left( \lceil \frac{p}{|t_x|} \rceil  -q \right)|t_x| + |t_x| - p_x$, where $p_x = |t_x|\lceil \frac{p}{|t_x|} \rceil -p$. We define $r_y$ in the same manner. A similar proposition holds if $t_y \neq 0$.
    \end{prop}

    \begin{proof}
      The proof is postponed to Appendix A.
    \end{proof}

    This property allows us to compute exactly the eigenvalues appearing in Proposition \ref{prop:squared_exact}. In Figure~\ref{fig:thres_wn} we illustrate that $a(\veclet{t})$ for fixed patch size ($8 \times 8$) and patch search window ($21 \times 21$). Thus in our implementation we suppose that $a(\veclet{t}) = \autoprob^{-1}\left(\veclet{t}, \omega, 1 - \NFAmaxmath / |T|\right)$ is constant and set its value to the mean of $a(\veclet{t})$ over $\veclet{t} \in T$.

    \begin{figure}
      \label{eigenvalues}
      \centering
  \subfloat[]{\includegraphics[width=0.4\linewidth]{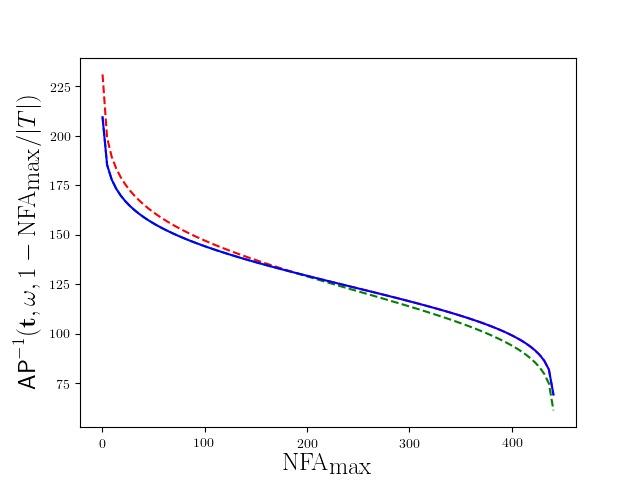}} \qquad
  \subfloat[]{\includegraphics[width=0.4\linewidth]{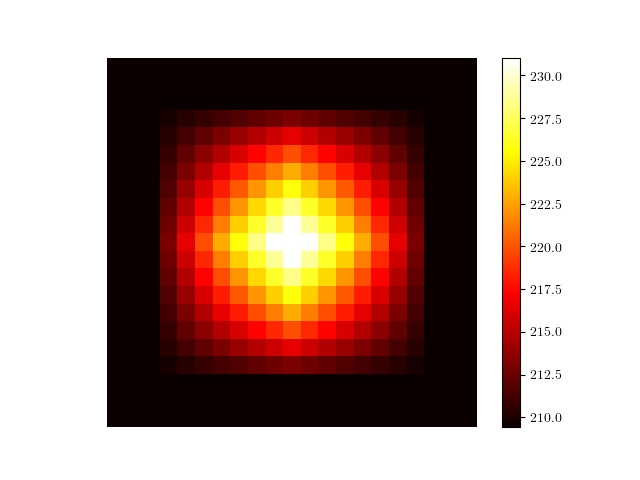}} \hfill
  \caption{\figuretitle{Thresholds dependency in $\NFAmaxmath$} In (a) we display $a(\veclet{t}) =\autoprob^{-1}\left(\veclet{t}, \omega, 1 - \NFAmaxmath / |T|\right)$ as a function of $\NFAmaxmath$. The patch size is fixed to $8 \times 8$ and the offsets $\veclet{t}$ satisfy $\| \veclet{t} \|_{\infty} \leq 10$, hence $|T| = 21^2=441$. The red dashed line is given by $\max_{\veclet{t} \in T}a(\veclet{t})$ and the green dashed line by $\min_{\veclet{t} \in T}a(\veclet{t})$. The blue line represents the value obtained considering the simplifying assumption that patch domains do not overlap, see Proposition \ref{prop:squared_exact} and the remark that follows
    . The maximal increase between the maximum of $a(\veclet{t})$ and the minimum of $a(\veclet{t})$ is of $13.0 \%$. In (b) we display the mapping $\veclet{t} \ \mapsto \ a(\veclet{t})$ for $\NFAmaxmath = 0.5$, the central pixel corresponds to $\veclet{t=0}$. Note that 
        $a(\veclet{t})$ decreases as $\|\veclet{t}\|$ increases and is constant when, $\|\veclet{t}\|_{\infty} \geq  8$.}
      \label{fig:thres_wn}
    \end{figure}

    \subsection{Some experimental results}
    \label{sec:expe-results}
    
    In the following paragraph we present and comment some results of our threshold NL-means algorithm, see Algorithm \ref{alg:nlmeans_thres}. We recall that we use $a(\veclet{t}) = \sum_{\veclet{t} \in T} \autoprob^{-1}\left(\veclet{t}, \omega, 1 - \NFAmaxmath / |T|\right) / |T|$. In Figure \ref{fig:denoised res} we present a first comparison with the NL-means algorithm. Perceptual results as well as Peak Signal to Noise Ratio (\PSNR ) measurements \footnote{$\operatorname{PNSR}(u,v) = 10 \log_{10} \left( \frac{\max_{\Omega} u^2}{\| u - v\|_2^2}\right) \eqsp .$} are commented. We also present the running time of the original NL-means algorithm and ours. The experiments were conducted with the following computer specifications: 16G RAM, 4 Intel Core i7-7500U CPU (2.70GHz). Results on other images than Barbara are displayed in Figure \ref{fig:nl_means_comp_vis}.
    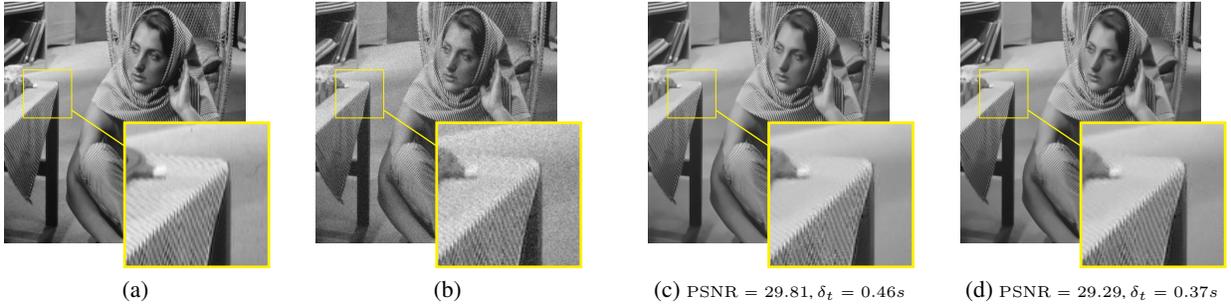
\begin{figure}
  \centering
  \subfloat[]{\resizebox{.23\textwidth}{!}{\begin{tikzpicture}[spy using outlines={rectangle, yellow,magnification=3, connect spies}]
  \node {\pgfimage[interpolate=true,height=2.5cm]{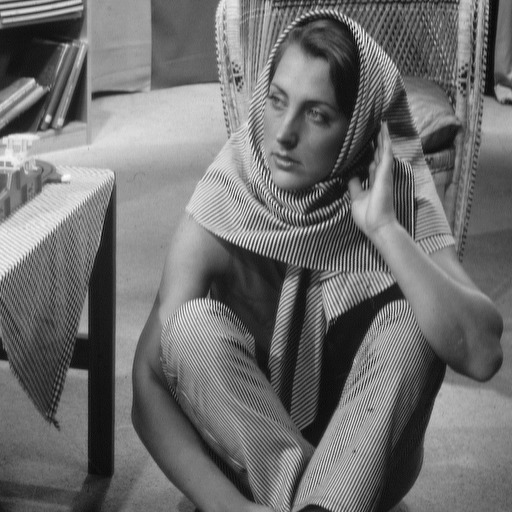}};
 \coordinate (spypoint) at (-0.8, 0.3);
 \coordinate (spyviewer) at (0.75,-0.75);
 \spy[width=1.5cm,height=1.5cm] on (spypoint) in node [fill=white] at (spyviewer);
  
\end{tikzpicture}
}} \quad
  \subfloat[]{\resizebox{.23\textwidth}{!}{\begin{tikzpicture}[spy using outlines={rectangle, yellow,magnification=3, connect spies}]
  \node {\pgfimage[interpolate=true,height=2.5cm]{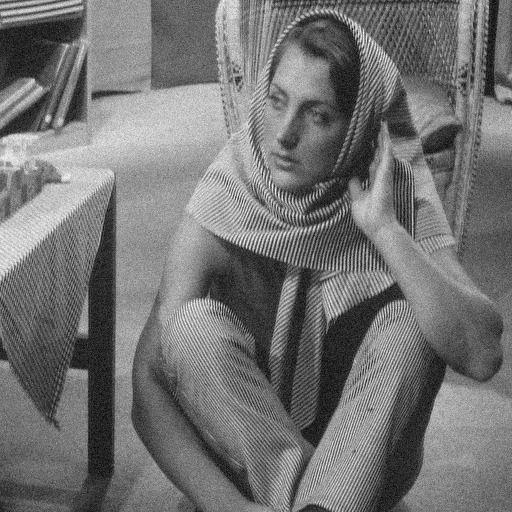}};
 \coordinate (spypoint) at (-0.8, 0.3);
 \coordinate (spyviewer) at (0.75,-0.75);
 \spy[width=1.5cm,height=1.5cm] on (spypoint) in node [fill=white] at (spyviewer);
  
\end{tikzpicture}
}} \hfill 
  \subfloat[\tiny $\operatorname{PSNR} = 29.81$, $\delta_t = 0.46s$]{\resizebox{.23\textwidth}{!}{\begin{tikzpicture}[spy using outlines={rectangle, yellow,magnification=3, connect spies}]
  \node {\pgfimage[interpolate=true,height=2.5cm]{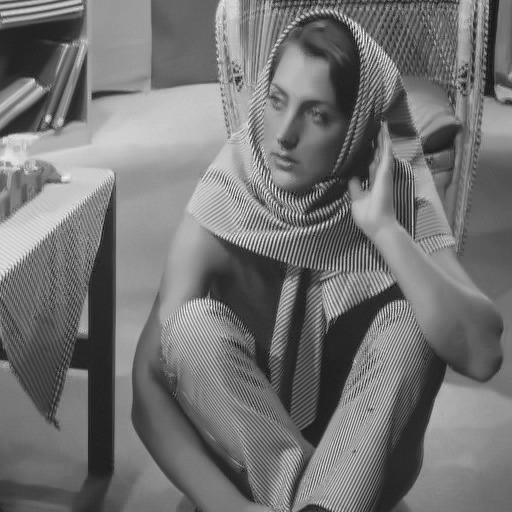}};
 \coordinate (spypoint) at (-0.8, 0.3);
 \coordinate (spyviewer) at (0.75,-0.75);
 \spy[width=1.5cm,height=1.5cm] on (spypoint) in node [fill=white] at (spyviewer);
  
\end{tikzpicture}
}} \quad
  \subfloat[\tiny $\operatorname{PSNR} = 29.29$, $\delta_t = 0.37s$]{\resizebox{.23\textwidth}{!}{\begin{tikzpicture}[spy using outlines={rectangle, yellow,magnification=3, connect spies}]
  \node {\pgfimage[interpolate=true,height=2.5cm]{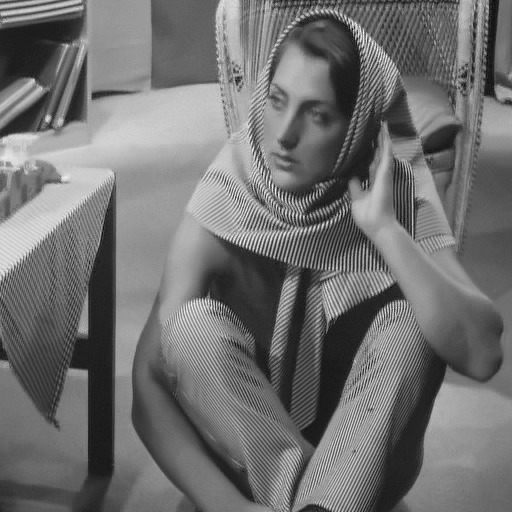}};
 \coordinate (spypoint) at (-0.8, 0.3);
 \coordinate (spyviewer) at (0.75,-0.75);
 \spy[width=1.5cm,height=1.5cm] on (spypoint) in node [fill=white] at (spyviewer);
  
\end{tikzpicture}
}} \hfill
  \caption{\figuretitle{Visual results} In (a) we present an original image (Barbara) scaled between $0$ and $255$. In (b) we add Gaussian white noise with $\sigma = 10$. We recall that the patch domain is fixed to $\omega_0$ being a $8 \times 8$ square. In (c) we present the denoised results with NL-means threshold, Algorithm \ref{alg:nlmeans_thres}, where $\NFAmaxmath = 4.41$, which corresponds to 1\% of rejected patches in the search window of a Gaussian white noise. In (d) we present the results obtained with the traditional NL-means algorithm with $h = 0.13 \sigma |\omega|$ (optimal $h$ for this noise level and this image with regard to the \PSNR \ measure). The results are the same on the texture area for (c) and (d). The perceptual results on the zoomed region are satisfying, even though some regions are too smooth compared to the original image (a). 
    In (c) and (d), $\delta_t$ is the running time of the algorithm. We can observe that our algorithm is slightly slower than NL-means.}
  \label{fig:denoised res}
\end{figure}

\captionsetup[subfigure]{labelformat=empty}
\begin{figure}
  \centering
  \subfloat[]{\includegraphics[width=0.22\linewidth]{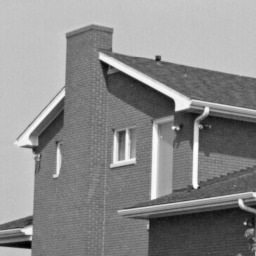}} \hfill
  \subfloat[]{\includegraphics[width=0.22\linewidth]{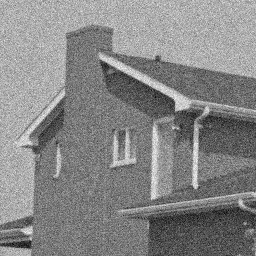}} \hfill
  \subfloat[\scriptsize $\operatorname{PSNR} =31.67$, $\delta_t = 0.21s$]{\includegraphics[width=0.22\linewidth]{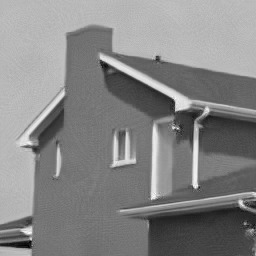}} \hfill
  \subfloat[\scriptsize $\operatorname{PSNR} =30.81$, $\delta_t = 0.07s$]{\includegraphics[width=0.22\linewidth]{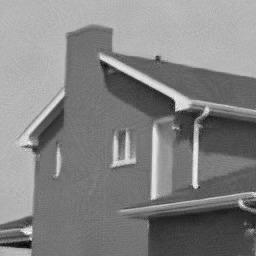}} \hfill \\
  \subfloat[]{\includegraphics[width=0.22\linewidth]{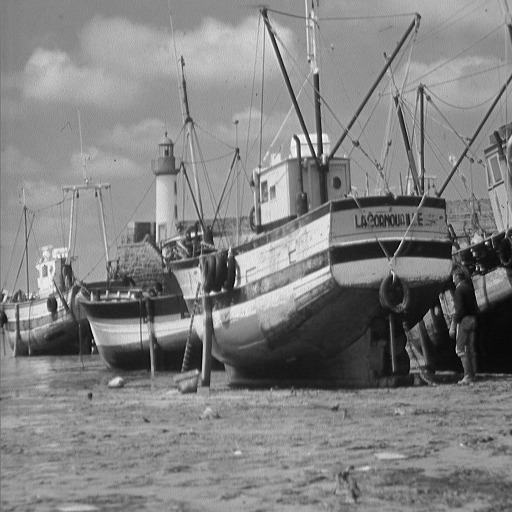}} \hfill
  \subfloat[]{\includegraphics[width=0.22\linewidth]{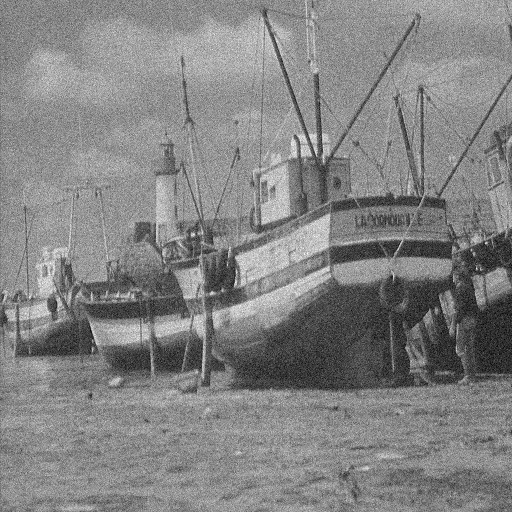}} \hfill
  \subfloat[\scriptsize $\operatorname{PSNR} =29.12$, $\delta_t = 0.46s$]{\includegraphics[width=0.22\linewidth]{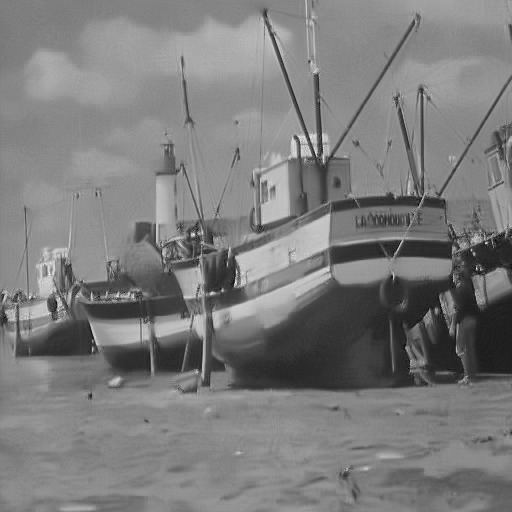}} \hfill
  \subfloat[\scriptsize $\operatorname{PSNR} =28.44$, $\delta_t = 0.39s$]{\includegraphics[width=0.22\linewidth]{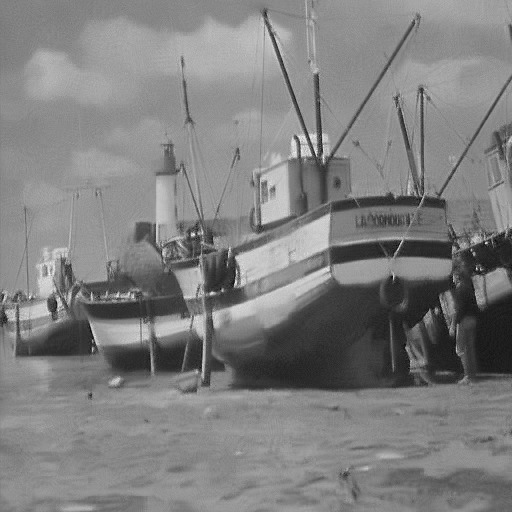}} \hfill \\
  \subfloat[]{\includegraphics[width=0.22\linewidth]{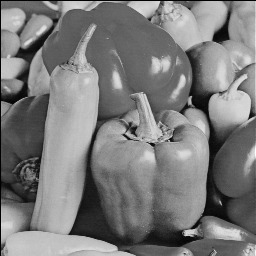}} \hfill
  \subfloat[]{\includegraphics[width=0.22\linewidth]{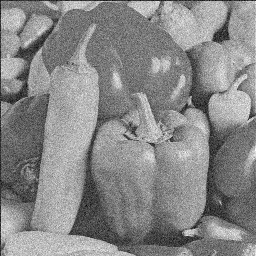}} \hfill
  \subfloat[\scriptsize $\operatorname{PSNR} =29.43$, $\delta_t = 0.22s$]{\includegraphics[width=0.22\linewidth]{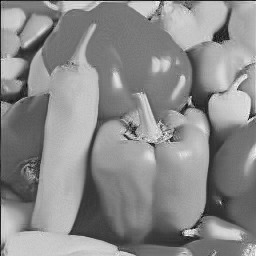}} \hfill
  \subfloat[\scriptsize $\operatorname{PSNR} =29.03$, $\delta_t = 0.07s$]{\includegraphics[width=0.22\linewidth]{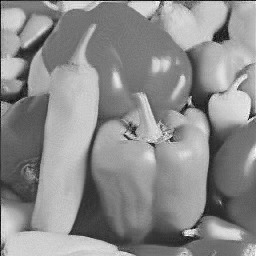}} \hfill \\
  \subfloat[]{\includegraphics[width=0.22\linewidth]{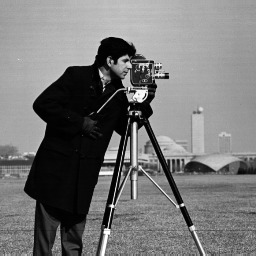}} \hfill
  \subfloat[]{\includegraphics[width=0.22\linewidth]{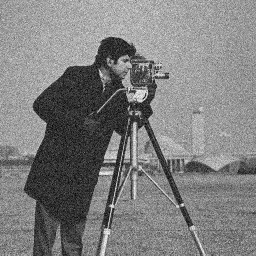}} \hfill
  \subfloat[\scriptsize $\operatorname{PSNR} =28.82$, $\delta_t = 0.22s$]{\includegraphics[width=0.22\linewidth]{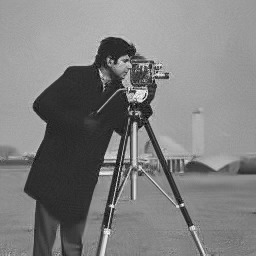}} \hfill
  \subfloat[\scriptsize $\operatorname{PSNR} =28.68$, $\delta_t = 0.09s$]{\includegraphics[width=0.22\linewidth]{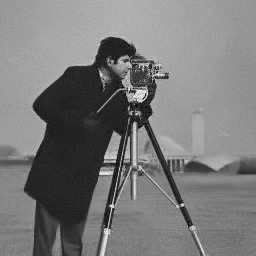}} \hfill \\  
  \caption{\figuretitle{NL-means comparison} In this figure we compare Algorithm \ref{alg:nlmeans_thres} with the traditional NL-means algorithm. Here $\omega_0$ is fixed to be a $8 \times 8$ square. The first column contains clean images, the second column represents the same images corrupted by a Gaussian noise with $\sigma = 20$. The third column is the output of Algorithm \ref{alg:nlmeans_thres} with $\NFAmaxmath$ fixed to $4.41$ and the last column is the output of the NL-means algorithm for the optimal value of $h$ (with regards to the PSNR), see \eqref{eq:original_nlmeans}. Perceptual results and \PSNR \ are comparable, even though our algorithm yields slightly better \PSNR \ values. We also present the running times $\delta_t$ of both algorithm. Our algorithm is slower than NL-means as it computes the threshold before running the NL-means algorithm.}
  \label{fig:nl_means_comp_vis}
\end{figure}
\captionsetup[subfigure]{labelformat=parens}

If the threshold $a(\veclet{t})$ is high, \ie \ $\NFAmaxmath \ll |T|$ then almost no patch is rejected, which means that almost all patches are used in the denoising process. In consequence, the output denoised image is very smooth. This smoothness is a correct guess for constant patches. However, this proposition does not hold when the region contains details. Indeed, in this case details are lost due to the averaging process. By setting a conservative threshold, \eg \ $\NFAmaxmath / |T| \approx 0.1$, for example, we reject all the patches for which the structure does not strongly match the one of the input patch, see Figure \ref{fig:number}. This conservative property of the algorithm ensures that we can control the loss of information in the denoised image, see Proposition \ref{prop:reconstruction}. However, if no patch, other than the input patch itself, is detected as similar we highly overfit the original noise. Many algorithms such as BM3D, see \cite{dabov2007image}, solve this problem by treating this case as an exception, applying a specific denoising method in this situation. 
We show the differences between our version of NL-means and BM3D in Figure \ref{fig:bm3D} .

\begin{figure}
  \centering
  \subfloat[$\NFAmaxmath / |T| = 0.2$]{\includegraphics[width=.28\linewidth]{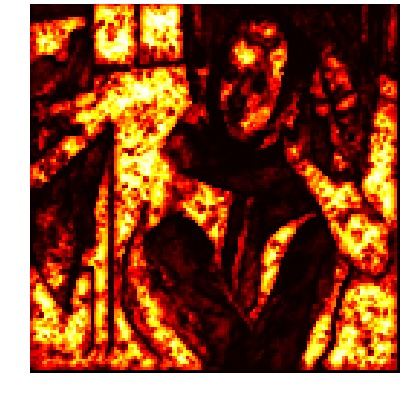}} \hfill  
  \subfloat[$\NFAmaxmath / |T| = 0.1$]{\includegraphics[width=.28\linewidth]{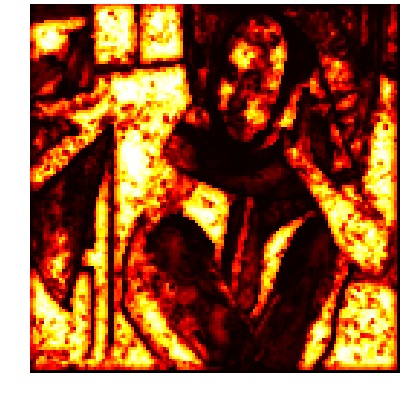}} \hfill
  \subfloat[$\NFAmaxmath / |T| = 0.01$]{\includegraphics[width=.33\linewidth]{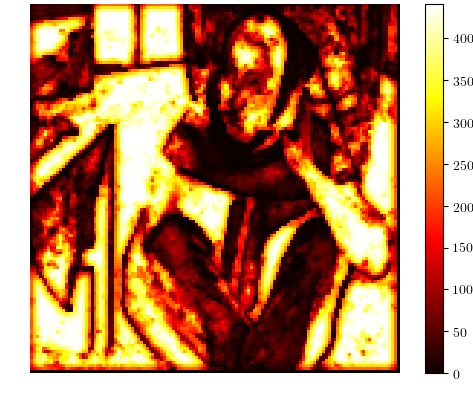}} \hfill
  \caption{\figuretitle{Number of detections} In this figure we present, for each denoised pixel, the number of detected offsets used to compute the denoised patch, \ie \ the cardinality of $\hat{T}$, see Proposition \ref{prop:reconstruction}. A white pixel means that the number of detected offsets is maximal and a black pixel means that the number of detected offsets is $1$, \ie \ the patch is not denoised. As $\NFAmaxmath$ decreases the number of detected offsets increases. Note that $|\hat{T}|$ is maximal, \ie \ equals to $21^2 = 441$, for constant regions. For $\NFAmaxmath / |T| = 0.1$, pixels located in textured neighborhoods use in average $20$ to $40$ patches to perform denoising.}
  \label{fig:number}
\end{figure}

\begin{figure}
  \centering
  \subfloat[original]{\includegraphics[width=0.24\linewidth]{./img/barbara.jpg}} \hfill
  \subfloat[BM3D]{\includegraphics[width=0.24\linewidth]{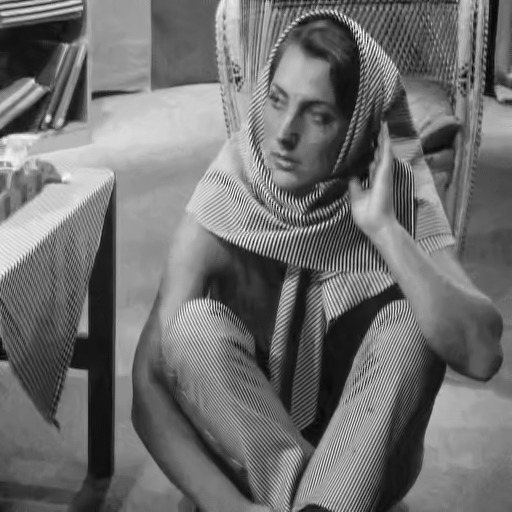}} \hfill
  \subfloat[$\NFAmaxmath / |T| = 0.01$]{\includegraphics[width=0.24\linewidth]{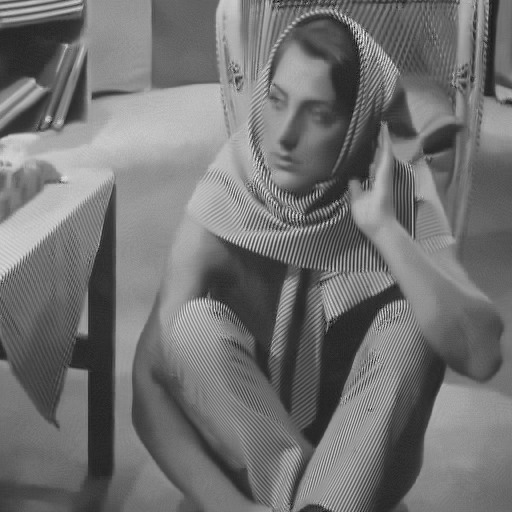}} \hfill
  \subfloat[$\NFAmaxmath / |T| = 0.1$]{\includegraphics[width=0.24\linewidth]{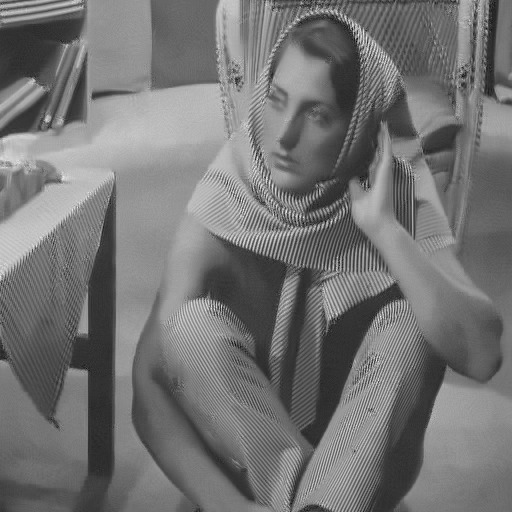}} \hfill
  \caption{\figuretitle{Comparison with BM3D} We compare Algorithm \ref{alg:nlmeans_thres} to BM3D \cite{dabov2007image}. The original image (Barbara) is presented in (a). We consider a noisy version of the input image with $\sigma =20$. In (b) we present the ouput of BM3D, with default parameters, see \cite{lebrun2012analysis}. The result in (c) corresponds to the output of Algorithm \ref{alg:nlmeans_thres} with $\NFAmaxmath / |T| = 0.01$. The output (c) is too blurry compared to (b). In order to correct this behavior we set $\NFAmaxmath / |T| = 0.1$ in (d), \ie \ increase the global threshold and some improvements are noticeable. 
    However the image remains blurry and artifacts due to the overfitting of the noise appear, this is known as the \textit{rare patch effect} in \cite{duval2011bias}. For instance, some patches in the scarf are not denoised anymore.}
  \label{fig:bm3D}
\end{figure}
In Figure \ref{fig:PSNR}, we show that Algorithm \ref{alg:nlmeans_thres} performs better than the original NL-means algorithm. By setting $\NFAmaxmath / |T| = 0.01$ we obtain that the \PSNR \ of the denoised image is better than the one of NL-means for nearly every value of $h$.
\begin{figure}
  \centering
  \subfloat[$\sigma = 10$]{\includegraphics[width=.333\linewidth]{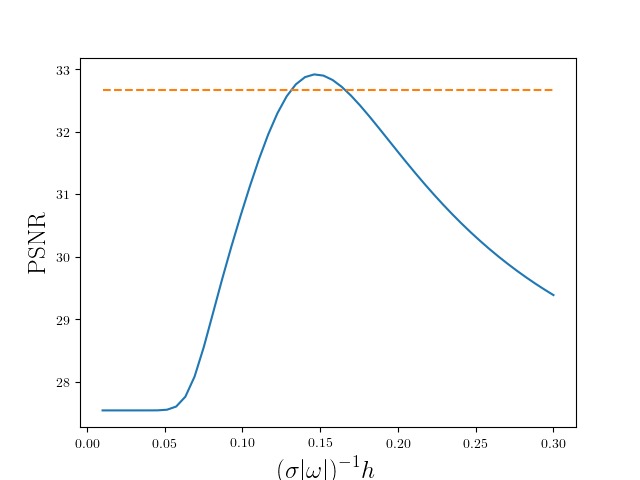}} \hfill  
  \subfloat[$\sigma = 20$]{\includegraphics[width=.333\linewidth]{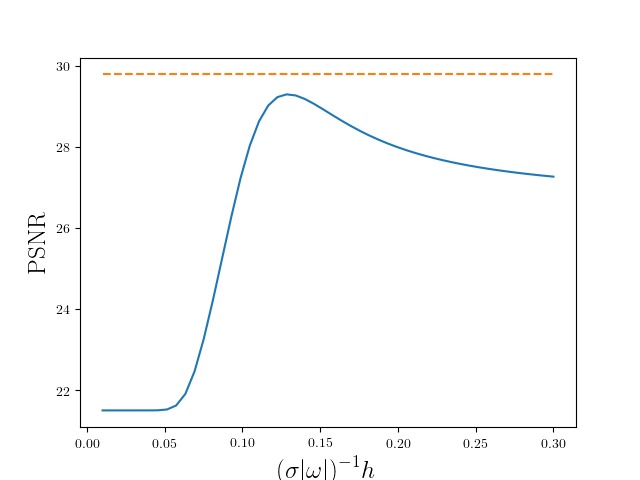}} \hfill
  \subfloat[$\sigma = 40$]{\includegraphics[width=.333\linewidth]{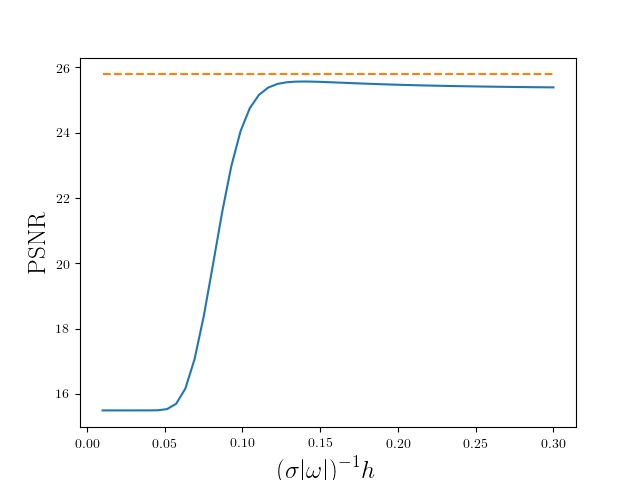}} \hfill
  \caption{\figuretitle{\PSNR \ study} In this figure we present the evolution of the \PSNR \ for different values of the parameter $h$ of the original NL-means method, see \eqref{eq:original_nlmeans}, in blue, computed on the Barbara image. The $x$-axis represents $\frac{h}{\sigma | \omega |}$. The orange dashed line is the \PSNR \ obtained for the threshold NL-means algorithm (Algorithm~\ref{alg:nlmeans_thres}) with $\NFAmaxmath / |T| = 0.01$. Except for low levels of noise the proposed method gives better \PSNR \ values than the original implementation of NL-means algorithm for every choice of $h$.}
  \label{fig:PSNR}  
\end{figure}

Let us emphasize that our goal is not to provide a new state-of-the-art denoising algorithm. Indeed we never obtain better denoising results than the BM3D algorithm. However, our algorithm slightly improves the original NL-means algorithm. It shows that statistical testing can be efficiently used to measure the similarity between patches and therefore provides a robust way to perform the weighted average in this algorithm.

\section{Periodicity analysis}
\label{sec:application to periodicity analysis}
\subsection{Existing algorithms}
\label{sec:existing_algorithms}
In the following sections we use our patch similarity detection algorithm, see Algorithm \ref{alg:auto-similaritydetection}, to analyze images exhibiting periodicity features. Let $\Omega \subset \Z^2$ be a finite domain and $\omega \subset \Omega$ a finite patch domain.

Periodicity detection is a long-standing problem in texture analysis \cite{zucker1980finding}. First algorithms used the quantization of images, relying on co-occurrence matrices and statistical tools like $\chi_2$ tests or $\kappa$ tests. Global methods extract peaks in the frequency domain (Fourier spectrum)~\cite{matsuyama1983structural} or in the spatial domain (autocorrelation). 
In \cite{haralick1973textural} the notion of inertia is introduced. It is defined for any $\veclet{t} \in \Omega$ by
$\mathcal{I}(\veclet{t}) = \sum_{(i,j) \in \llbracket 0,N_g \rrbracket^2}(i-j)^2 \left(\sum_{\veclet{z} \in \Omega}\mathbb{1}_{\dot{u}(\veclet{z}) = i}\mathbb{1}_{\dot{u}(\veclet{z+t}) = j}\right)$, 
where $u$ is a quantized image on $N_g+1$ gray levels. In \cite{conners1980toward}, the authors show that the local minima of the inertia measurement can be used to assess periodicity.
Similarly, we introduce the $\omega$-inertia for any $\veclet{t} \in \Omega$ by $\mathcal{I}_{\omega}(\veclet{t}) = \sum_{(i,j) \in \llbracket 0,N_g \rrbracket^2}(i-j)^2 \left(\sum_{\veclet{z} \in \omega}\mathbb{1}_{\dot{u}(\veclet{z}) = i}\mathbb{1}_{\dot{u}(\veclet{z+t}) = j}\right)$.
  The following proposition 
  extends to a local framework results from \cite{oh1999fast}.
\begin{prop}
  \label{prop:cooccurence}
  Let $u \in \R^{\Omega}$. Suppose that $u$ is quantized, \ie \ there exists $N_g \in \N$ such that for any $\veclet{x} \in \Omega$, $u(\veclet{x}) \in \llbracket 0,N_g \rrbracket$. We have $\mathcal{I}_{\omega}(\veclet{t}) = \autosim(u,\veclet{t},\omega)$.
\end{prop}

\begin{proof}
  For any $\veclet{t} \in \Omega$ we have
  \[\al{\mathcal{I}_{\omega}(\veclet{t}) &= \summ{(i,j) \in \llbracket 0,N_g \rrbracket^2}{}{(i-j)^2\summ{\veclet{x} \in  \omega}{}{\mathbb{1}_{\dot{u}(\veclet{x}) = i}\mathbb{1}_{\dot{u}(\veclet{x+t}) = j}}} = \summ{\veclet{x} \in \omega, (i,j) \in \llbracket 0,N_g \rrbracket^2}{}{(i-j)^2 \mathbb{1}_{\dot{u}(\veclet{x}) = i}\mathbb{1}_{\dot{u}(\veclet{x+t}) = j}}\\  &= \summ{\veclet{x} \in \omega}{}{(\dot{u}(\veclet{x}) - \dot{u}(\veclet{x+t}))^2} = \autosim(u,\veclet{t},\omega).}\]
\end{proof}

If $\omega = \Omega$ then the $\omega$-inertia statistics is exactly the inertia introduced in \cite{haralick1973textural} and the result is due to \cite{oh1999fast}. 
\subsection{Algorithm and properties}
\label{sec:algorithm and properites}
Lattice detection is closely related to periodicity analysis, since identifying a lattice is similar to extracting periodic or pseudo-periodic structures up to deformations and approximations. A state-of-the-art algorithm proposed in \cite{park2009deformed} uses a recursive framework which consists in 1) a lattice model proposal based on detectors such as Kanade-Lucas-Tomasi (\KLT ) feature trackers \cite{lucas1981iterative},  2) spatial tracking using inference in a probabilistic graphical model, 3) spatial warping correcting the lattice deformations in the original image. In this section we propose a new algorithm for lattice detection. 
The lattice proposal step 1) is replaced by an Euclidean auto-similarity matching detection (see Section \ref{sec:detection-algorithm} and Algorithm~\ref{alg:auto-similaritydetection}) where the patch domain $\omega$ is fixed. Using these detections we build a graph with a few nodes (usually approximately $20$ nodes for a $256 \times 256$ image). We use the same notation for the detection mapping $\veclet{t} \mapsto \mathbb{1}_{\autosim_i(u,\veclet{t}, \omega) \leq a(\veclet{t})}$ \ie \ the $D_{map}$ output of Algorithm \ref{alg:auto-similaritydetection}, which is a binary function over the offsets, and the set of detected offsets. We recall that two pixel coordinates $\veclet{x}$ and $\veclet{y}$ are said to be $8$-connected if $\veclet{x} = \veclet{y} + (\delta_x, \delta_y)$ with $\delta_x, \delta_y \in \lbrace -1, 0, 1 \rbrace$. The graph $\mathscr{G} = (V,E)$ is then built as follows:

\begin{itemize}[label = $\blacktriangleright$]
\item \textbf{Vertices}: for each 8-connected component, $\mathscr{C}_k$ in $D_{map}$ we note $\veclet{v}_k$ one position for which the minimum of $\autosim(u,\veclet{t},\omega)$ over $\mathscr{C}_k$ is achieved. The set of vertices $V$ is defined as $V= \seq[1][N_{\mathscr{C}}]{\veclet{v}}{k}$ where $N_{\mathscr{C}}$ is the number of 8-connected components in $D_{map}$ ;
\item \textbf{Edges}: each vertex is linked with its four nearest neighbors in the sense of the Euclidean distance, defining four unoriented edges. 
\end{itemize}

Refering to the three steps of \cite{park2009deformed} we present our model to replace step 2) (\ie \ the inference in a probabilistic graphical model) and introduce the approximated lattice hypothesis defined on a graph. 
\begin{mydef}[Approximated lattice hypothesis]
  \label{def:approximated_lattice_hypothesis}
  Let $\mathscr{G} = (V,E)$ be a random graph with $V \subset \R^2$. We say that $\mathscr{G}$ follows the approximated lattice hypothesis if there exists a basis $B = (b_1,b_2)$ of $\R^2$ and, for each edge $\veclet{e} \in E$, a couple of integers $(m_{\veclet{e}},n_{\veclet{e}}) \in \Z^2$ such that $\veclet{e}$ has the same distribution as 
  $m_{\veclet{e}} b_1 + n_{\veclet{e}} b_2 + \sigma Z_{\veclet{e}}$,
  with $(Z_{\veclet{e}})_{\veclet{e} \in E}$ independent standard Gaussian random variables in $\R^2$ and $\sigma >0$. We denote by $M$ the vector of all coefficients $(m_{\veclet{e}},n_{\veclet{e}})_{\veclet{e} \in E}  \in\Z^{2\vertt{E}} $.
\end{mydef}
Our goal is to perform inference in the statistical model defined by the following log-posterior
\begin{equation}
  \mathscr{L}(B,M,\sigma^2 |E) = -2(\vertt{E} +1)\log(\sigma^2)  -\frac{1}{2\sigma^2}\underbrace{\left( \summ{\veclet{e} \in E}{}{\| m_{\veclet{e}}b_1 + n_{\veclet{e}}b_2 - \veclet{\veclet{e}}\|^2} + r(B,M) \right)}_{q(B,M | E)} \eqsp ,
  \label{eq:log-lik}
\end{equation}
where $r(B,M) = \delta_B \|B\|_2^2 + \delta_M \| M \|_2^2$ with $\delta_B, \delta_M \geq 0$.
A discussion on the dependence of the model on the hyperparameters $(\delta_B, \delta_M)$ is conducted in Figure \ref{fig:hyperparam}.
Finding the \MLE \ of this full log-posterior is a non-convex, integer problem. However performing the minimization alternatively on $B$ and $M$ is easier since at each step we only have a quadratic function to minimize. Minimizing a positive-definite quadratic function over $\Z^2$ is equivalent to finding the vector of minimum norm in a lattice. 
This last formulation is known as the Shortest Vector Problem (\SVP ) which is a challenging problem \cite{micciancio2001svp} (though it is not known if it is a \NP -hard problem). We replace this minimization procedure over a lattice by a minimization over $\R^2$ followed by a rounding of this relaxed solution. 
\begin{algorithm}
  \caption{Lattice detection -- Alternate minimization
    \label{alg:alternate}}
  \begin{algorithmic}[1]
    \Function{Alternate-minimization}{$E$, $\delta_B$, $\delta_M$, $N_{it}$}
    \Let{$M_0$}{0}
    \Let{$B_0$}{initialization procedure} \Comment{initialization discussed in the end of Section \ref{sec:algorithm and properites}}
    \For{$n \gets 0 \ \textrm{to} \ N_{it} - 1$}
    \Let{$\tilde{M}$}{$\underset{ M \in \R^{2\vertt{E}}}{\operatorname{argmin}} \ q(B_n, M | E)$} \Comment{expression given in Proposition \ref{prop:alternate_update}}
    \If{ $q\left(E | B_n, [\tilde{M}]\right) <q \left(E | B_n, M_n\right) $} \Comment{$[\cdot]$ is the nearest integer operator}
    \Let{$M_{n+1}$}{$[\tilde{M}]$}
    \Else
    \Let{$M_{n+1}$}{$M_n$}
    \EndIf
    \Let{$B_{n+1}$}{$\underset{B \in \mathbb{R}^4}{\operatorname{argmin}} \  q(B, M_{n+1}|E)$} \Comment{expression given in Proposition \ref{prop:alternate_update}}
    \EndFor
    \Let{$\sigma_{N_{it}}^2$}{$\underset{\sigma^2 \in \mathbb{R_+}}{\operatorname{argmin}} \ -\mathscr{L}( B_{N_{it}},M_{N_{it}},\sigma^2| E)$}
    \State \Return{$B_{N_{it}}, M_{N_{it}}, \sigma_{N_{it}}^2$}
    \EndFunction
  \end{algorithmic}
\end{algorithm}

For any $\sigma >0$ we denote by $\mathscr{L}_n(\sigma) = \mathscr{L}(B_n,M_n,\sigma^2| E)$, with $n \in \N$, the log-posterior sequence.

\begin{prop}[Alternate minimization update rule]
  \label{prop:alternate_update}
  In Algorithm \ref{alg:alternate}, we get for any $n \in \N$
  \[ \tilde{M} = \left( \Lambda_{B_n} \otimes \operatorname{Id}_{\vertt{E}}  \right)^{-1} E_{B_n} \in \R^{2 \vertt{E}} \eqsp , \qquad 
        B_{n+1} = \left( \Lambda_{M_{n+1}} \otimes \operatorname{Id}_{2} \right)^{-1} E_{M_{n+1}}\in \R^4 \eqsp ,
  \]
  with $\otimes$ the tensor product between matrices and
  \begin{enumerate}[label=(\alph*)]
\item $\Lambda_B = \left( \begin{matrix} \|b_1\|^2+\delta_B & \langle b_1,b_2 \rangle \\ \langle b_1,b_2 \rangle & \|b_2\|^2+\delta_B \end{matrix} \right) \eqsp , \qquad 
      \Lambda_M = \left( \begin{matrix} \|M_1\|^2+\delta_M & \langle M_1,M_2 \rangle \\ \langle M_1,M_2 \rangle & \|M_2\|^2+\delta_M \end{matrix} \right) \eqsp ;$
       \item $E_B = \left( \begin{matrix} (\langle \veclet{e}, b_1 \rangle)_{\veclet{e} \in E} \\ (\langle \veclet{e}, b_2 \rangle)_{\veclet{e} \in E} \end{matrix} \right) \eqsp , \qquad
      E_M = \left( \begin{matrix} \summ{\veclet{e} \in E}{}{m_{\veclet{e}}\veclet{e}}  \\ \summ{\veclet{e} \in E}{}{n_{\veclet{e}} \veclet{e}}  \end{matrix}\right) \eqsp .$
\end{enumerate}
\end{prop}

\begin{proof}
  The proof is postponed to Appendix B.
\end{proof}

Note that if $B$ is orthogonal, \ie \ $\langle b_1, b_2 \rangle = 0$ then $\Lambda_B$ is diagonal and the proposed method is the exact solution to the minimization problem over $\Z^2$.
\begin{thm}[Convergence in finite time]
For any $\sigma >0$, $(\mathscr{L}_n(\sigma))_{n \in \N}$ is a non-decreasing sequence. In addition, $\seq{B}{n}$ and $\seq{M}{n}$ converge in a finite number of iterations.
\end{thm}

\begin{proof}
  $(\mathscr{L}_n(\sigma))_{n \in \N}$ is non-decreasing since for any $n\in \N$, $\mathscr{L}_n(\sigma) \leq  \mathscr{L}( B_{n},M_{n+1},\sigma^2 | E) \leq \mathscr{L}_{n+1}(\sigma)$. Let us show that the sequences $(M_n)_{n \in \N}$ and $(B_n)_{n \in \N}$ are bounded. Because $(\mathscr{L}_n(\sigma))_{n \in \N}$ is non-decreasing, the sequence $\left(q(B_n, M_n|E) \right)_{n \in \N}$ is non-increasing. We obtain that 
  \[ 
        \delta_M \| M_n \|^2 \leq q(B_0, M_0 | E) \eqsp , \qquad
        \delta_B \| B_n \|^2 \leq q(B_0, M_0 | E) \eqsp .
    \]  
  The sequence $\seq{M}{n}$ is bounded thus we can extract a converging subsequence. Since $\seq{M}{n}$ takes value in $\Z^{2 \vertt{E}}$, this subsequence is stationary with value $M$. Let $n_0 \in \N$ be the first time we hit value $M$. Let $n \in \N$, with $n \ge n_0+1$, there exists $n_1 \in \N$, with $n_1 \ge n$ such that $M_{n_1} = M_{n_0}$ thus
  \[ \mathscr{L}_{n_0}(\sigma) \leq \mathscr{L}_{n_0+1}(\sigma) \leq \mathscr{L}_n(\sigma) \leq \mathscr{L}(B_{n_1-1}, M_{n_1}, \sigma^2 | E) \leq \mathscr{L}(B_{n_1-1}, M_{n_0}, \sigma^2 | E) \leq \mathscr{L}_{n_0}(\sigma) \eqsp .\] Hence for every $n \ge n_0+1$, $\mathscr{L}_n(\sigma) = \mathscr{L}( B_n, M_n, \sigma^2 | E) = \tilde{\mathscr{L}}(\sigma)$. Suppose there exists $n \ge n_0+1$ such that $M_n \neq M_{n+1}$ this means that $\mathscr{L}(B_n,M_{n+1},\sigma^2|E) > \mathscr{L}_n(\sigma)$ (because of lines 6 and 7 of Algorithm \ref{alg:alternate}) which is absurd. Thus $\seq{M}{n}$ is stationary and so is $\seq{B}{n}$.
\end{proof}

In Algorithm \ref{alg:alternate} 
$M_0$ is initialized with zero and $B_0$ is defined as an orthonormal (up to a dilatation factor) direct basis where the first vector is given by an edge with median norm in $E$. 

\begin{figure}
  \centering
  \subfloat[$\delta_M =0$ $\delta_B =0$]{\includegraphics[width=.24\linewidth]{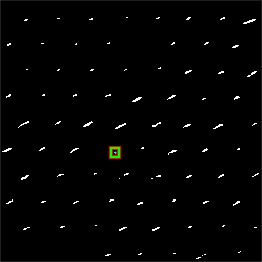}} \hfill
  \subfloat[$\delta_M = 5$ $\delta_B =10^{-1}$]{\includegraphics[width=.24\linewidth]{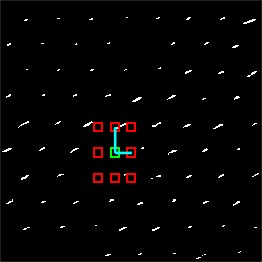}}  \hfill
  \subfloat[$\delta_M = 9$ $\delta_B =10^{-1}$]{\includegraphics[width=.24\linewidth]{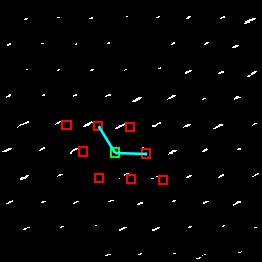}}  \\
  \caption{\figuretitle{Influence of hyperparameters} In this experiment we assess the importance of the hyperparameters. We consider Algorithm \ref{alg:alternate} with input graph a detection map, output of Algorithm \ref{alg:auto-similaritydetection}. The initialization in the three cases is the canonical basis $((0,1), (1,0))$. In (a), since the initial basis vectors are a local minimum to the optimization problem, the algorithm converges after one iteration. However, this is not perceptually satisfying. Setting $\delta_M = 5$ and $\delta_B = 10^{-1}$ in (b) the true observed lattice is a sub-lattice of the output lattice of Algorithm~\ref{alg:alternate}. Increasing $\delta_M$ up to 9, in (c) we obtain a perceptually correct lattice. For $\delta_M$ larger than 10, the basis vectors go to 0. Only the regularizing term is minimized by the optimization procedure and the data-attachment term is not taken into account. Experimentally we found that the choice of $\delta_M$ is more flexible and that $\delta_M \in (1,20)$ gives satisfying perceptual results if the initialization heuristics proposed in Section \ref{sec:algorithm and properites} is chosen.}
  \label{fig:hyperparam}
\end{figure}
\subsection{Experimental results}
\label{sec:experimental-results}
Combining the results of Section \ref{sec:algorithm and properites} and Section \ref{sec:detection-algorithm} we obtain an algorithm to extract lattices in images, see Figure \ref{fig:lattice_detec}. In what follows we perform lattice detection using Algorithm \ref{alg:auto-similaritydetection} in order to extract auto-similarity given a patch in an original image $u$, which implies that the patch domain $\omega$ is set by the user. Recall that in Algorithm~\ref{alg:auto-similaritydetection}, the eigenvalues of the covariance matrix in Proposition \ref{prop:squared_exact} are approximated, and that the cumulative distribution function of the quadratic form in Gaussian random variables is computed via the Wood F method \cite{wood1989f}. Lattice detection is performed using Algorithm~\ref{alg:alternate} with parameters $\delta_M = 10$ and $\delta_B = 10^{-2}$.
\begin{figure}
  \centering
  \input{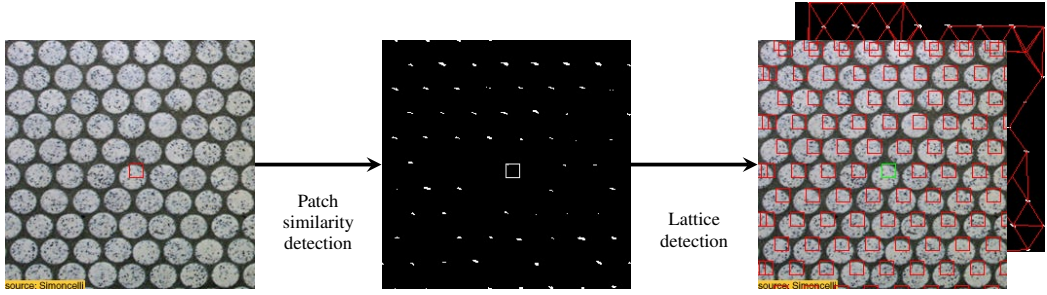}
  \caption{\figuretitle{Lattice proposal algorithm} Lattice detection and extraction in images require a patch from the user and compute a binary image containing all the offsets with correct similarity as well as a lattice matching the underlying graph. The patch auto-similarity detection step was presented in Section \ref{sec:detection-algorithm}. The lattice detection step was presented in Section \ref{sec:algorithm and properites}. The first image is the input, the second one is the output of the detection algorithm. In the last step we show the original image with red squares placed on the computed lattice. Behind this image, the unoriented edges of the graph are shown in red.}
  \label{fig:lattice_detec}
\end{figure}
\subsubsection{Escher paving}
\label{sec:escher_paving}
In this section we study art images, Escher pavings, with strongly periodic structure. We investigate the following parameters of our lattice detection algorithm:
\begin{enumerate}[label = (\alph*)]
\item background microtexture model $\micro$,
\item \NFAmax \ parameter in Algorithm \ref{alg:auto-similaritydetection},
\item patch domain $\omega$.
\end{enumerate}

\paragraph{Microtexture model}
We confirm that the choice of the microtexture model will influence the detected geometrical structures. 
The more structured is the background noise model the less we obtain detections. This situation is considered in Figure \ref{fig:microtexture_model}. 

\begin{figure}
  \centering
  \subfloat[]{\includegraphics[width=.24\linewidth]{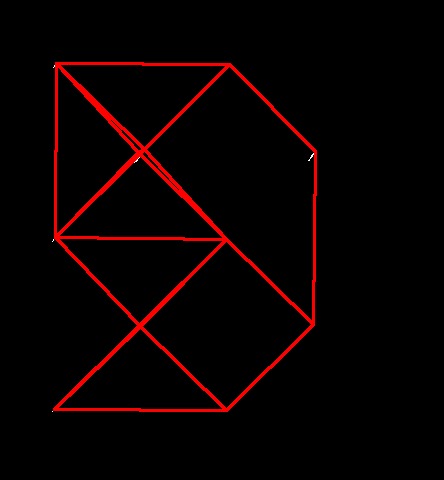}} \hspace{0.5cm}
  \subfloat[]{\includegraphics[width=.24\linewidth]{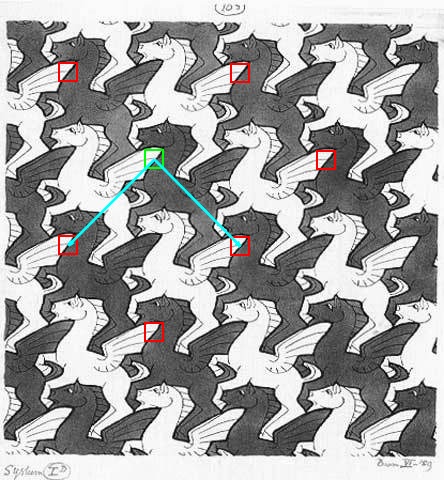}}  \hspace{0.5cm}
  \subfloat[]{\includegraphics[width=.24\linewidth]{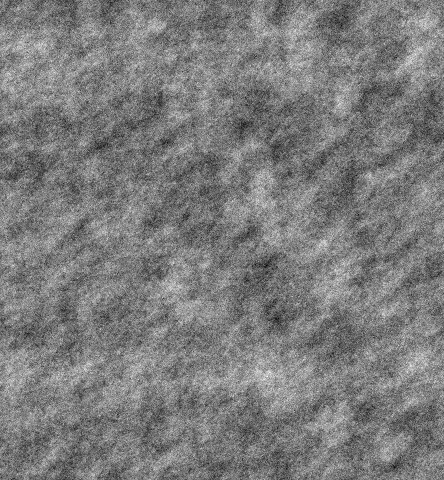}}  \\
  \subfloat[]{\includegraphics[width=.24\linewidth]{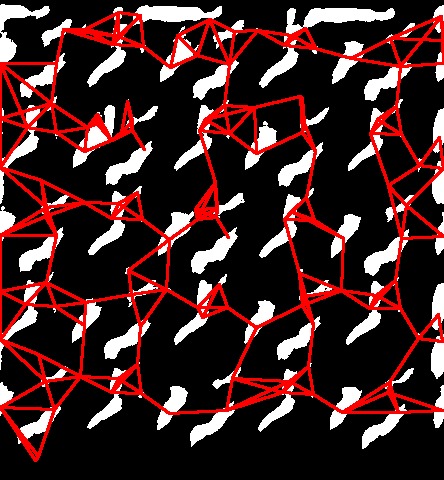}} \hspace{0.5cm}
  \subfloat[]{\includegraphics[width=.24\linewidth]{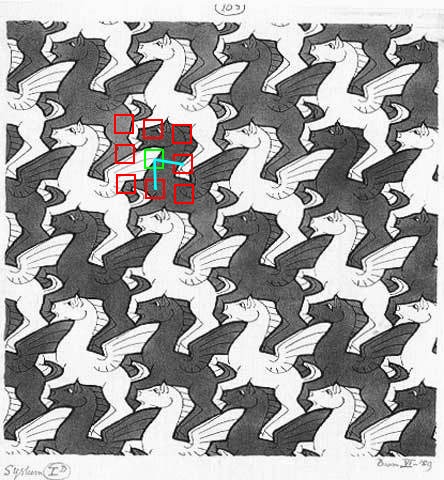}}  \hspace{0.5cm}
  \subfloat[]{\includegraphics[width=.24\linewidth]{./img/white_noise.jpg}}  \\  
  \subfloat[]{\includegraphics[width=.24\linewidth]{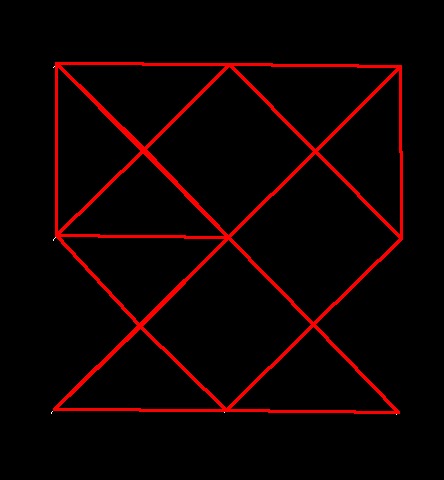}} \hspace{0.5cm}
  \subfloat[]{\includegraphics[width=.24\linewidth]{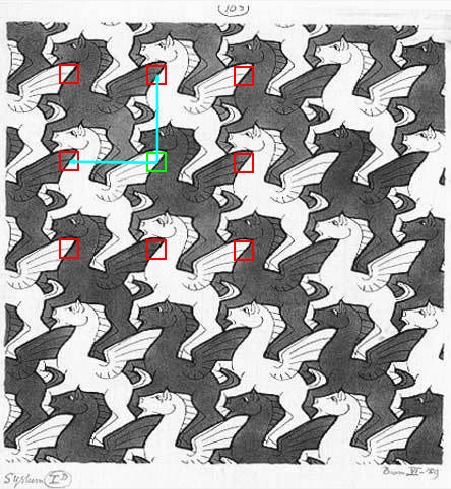}}    \hspace{0.5cm}
  \subfloat[]{\includegraphics[width=.24\linewidth]{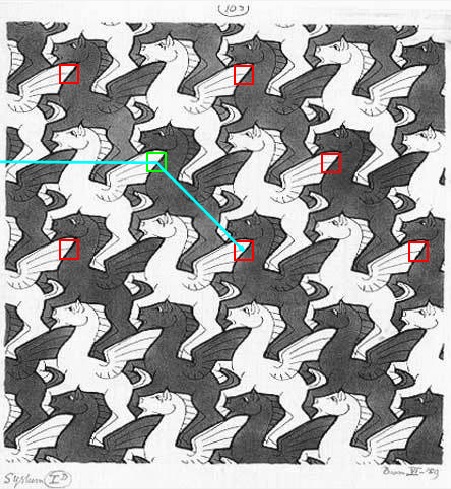}}  \\  
  \caption{\figuretitle{Choice of the microtexture model} In this experiment we discuss the choice of the \acontrario \ background microtexture model. In the left column we display the graph obtained after the detection step. In the middle column we superpose the proposed lattice on the original image. The original patch is drawn in green, obtained basis lattice vectors are in cyan, and red squares are placed onto the proposed lattice. In (a) and (b) the microtexture model is given by \eqref{eq:gaussian_model} 
    and \NFAmax \ is set to $10$. A sample of this model is presented in (c). Obtained results match the perceptual lattice. In (d), (e), (g) and (h) the microtexture model is a Gaussian white noise model with variance equal to the empirical variance of the original image. Sample from this Gaussian white noise is presented in (f). In (d) and (e), \NFAmax \ is set to $10$
    . This leads to an excessive number of detections in the input image. In order to obtain the perceptual lattice found in (b) with a Gaussian white noise model 
    we must set the \NFAmax \ parameter to $10^{-111}$. Results are presented in experiments (g), (h) and (i). Image (h) is also an example for which the median initialization for $B_0$ in Algorithm \ref{alg:alternate} identifies a non satisfying local minimum. 
    This situation is corrected in (i) with random initialization for $B_0$. In (h) final log-posterior value is $-565.5$ which is inferior to the final log-posterior value in (i): $-542.1$. Thus (i) gives a better local maximum of the full log-posterior than (h).}
  \label{fig:microtexture_model}
\end{figure}

\paragraph{\NFAmax \ parameter}
Using a more adapted microtexture model as background model we gain robustness compared to other less structured models such as a Gaussian white noise. However, \NFAmax \ must be set carefully otherwise two situations can occur:
\begin{enumerate}[label=(\alph*)]
\item if \NFAmax \ is too high, too many detections can be obtained (true perceptual detections are not differentiated from false positives) ;
\item if \NFAmax \ is too low, we fail to identify important perceptual structures in the image.
\end{enumerate}
We observe that a general good practice is to set \NFAmax \ equal to $10$, see Figure \ref{fig:NFA}. However, if the input patch is corrupted one may increase this parameter up to $10^2$ or $10^3$, see Figure \ref{fig:preprocessing} and Figure \ref{fig:homography}.

\begin{figure}
  \centering
  \subfloat{\includegraphics[width=.24\linewidth]{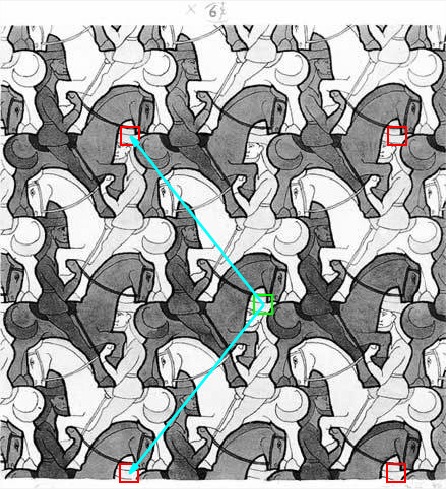}}  \hfill
  \subfloat{\includegraphics[width=.24\linewidth]{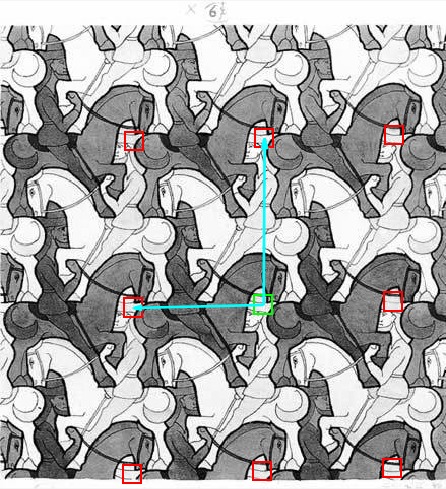}} \hfill
  \subfloat{\includegraphics[width=.24\linewidth]{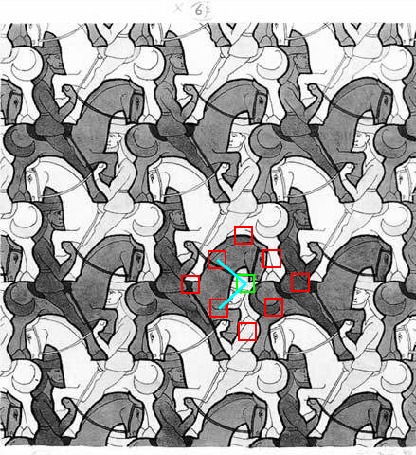}} \\
  \setcounter{subfigure}{0}
  \subfloat[]{\includegraphics[width=.24\linewidth]{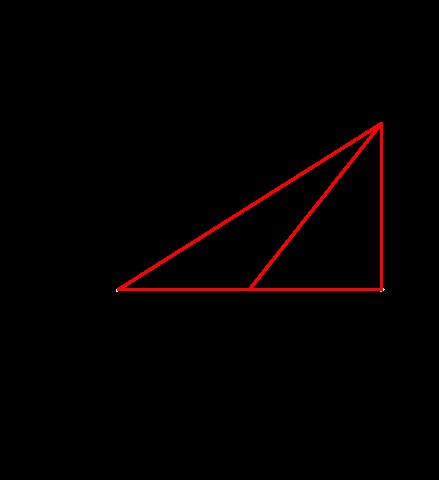}} \hfill
  \subfloat[]{\includegraphics[width=.24\linewidth]{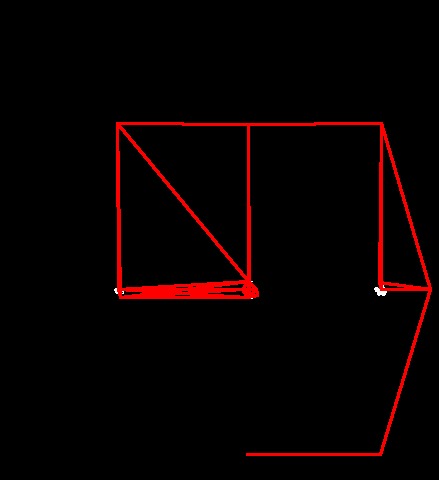}} \hfill  
  \subfloat[]{\includegraphics[width=.24\linewidth]{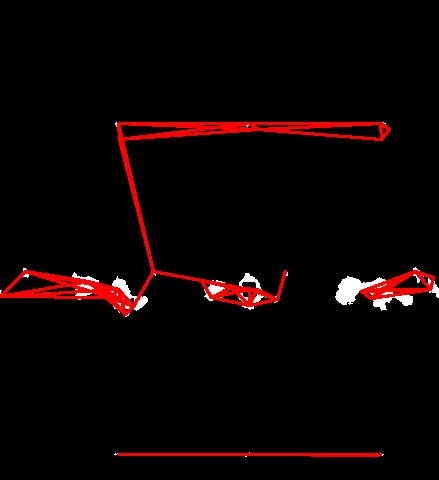}} \hfill  
  \caption{\figuretitle{Choice of Number of False Alarms} In this experiment we discuss the choice of the \NFAmax \ parameter in the \acontrario \ framework in the case where the underlying microtexture model is given by \eqref{eq:gaussian_model}
    . Each column corresponds to a pair of images: the returned lattice and its associated underlying graph. In (a), \NFAmax \ is set to 1. Detections are correct, there are not enough points to precisely retrieve the perceptual lattice. 
    In (b), \NFAmax \ is set to $10$. The estimated lattice is correct. In (c), \NFAmax \ is set to $10^3$. In this case we obtain false detections which lead to an incorrect final lattice. Note that large detection zones in the binary image (c) are due to the non-validity of the Wood F approximation for some offsets. This behavior is also present in (a) and (b) but less noticeable.}
  \label{fig:NFA}
\end{figure}

\paragraph{Patch position}
Patch position and size are crucial in our detection model, since we rely on local properties of the image. As shown in Figure \ref{fig:pos_size} these parameters should be carefully selected by the user. However, for particular applications such as lattice extraction for crystallographic purposes, there exist procedures to extract primitive cells \cite{mevenkamp2015unsupervised}.

\begin{figure}
  \centering
  \subfloat[]{\includegraphics[width=.24\linewidth]{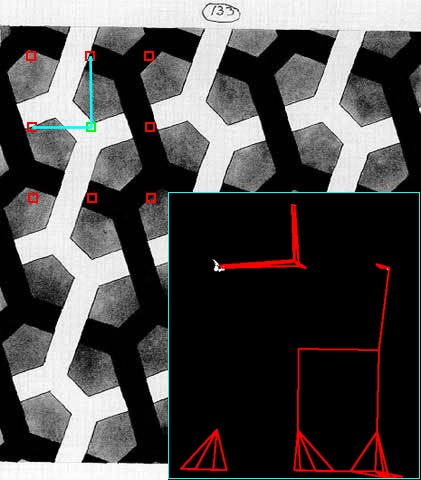}} \hspace{0.2cm}
  \subfloat[]{\includegraphics[width=.24\linewidth]{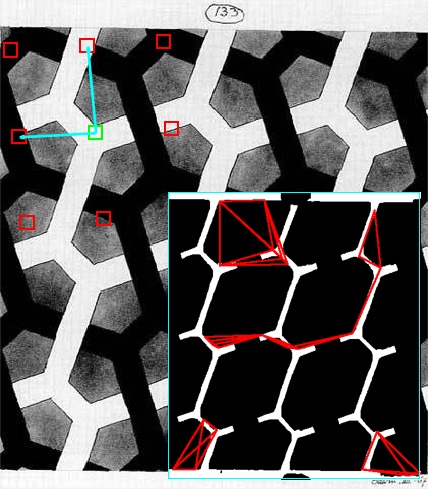}} \hspace{0.2cm}
  \subfloat[]{\includegraphics[width=.24\linewidth]{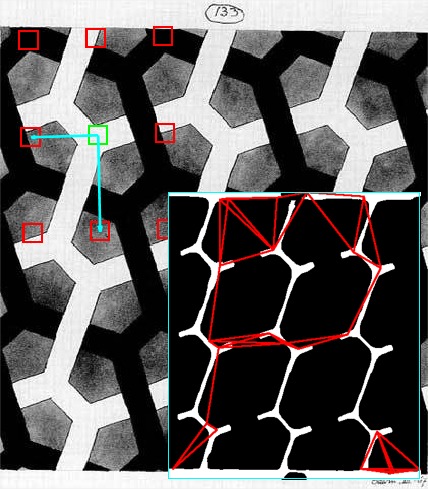}} \\
  \subfloat[]{\includegraphics[width=.24\linewidth]{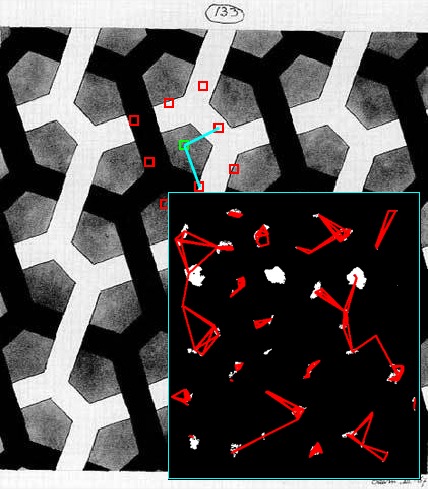}} \hspace{0.2cm}
  \subfloat[]{\includegraphics[width=.24\linewidth]{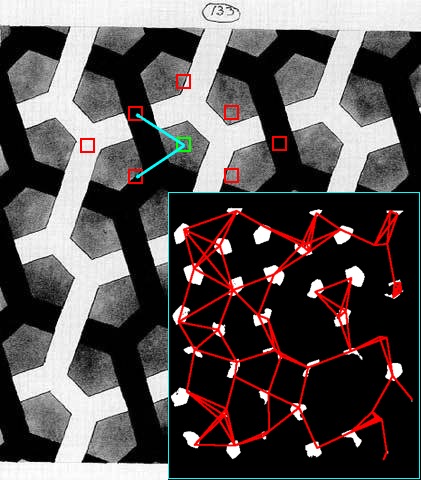}} \hspace{0.2cm}
  \subfloat[]{\includegraphics[width=.24\linewidth]{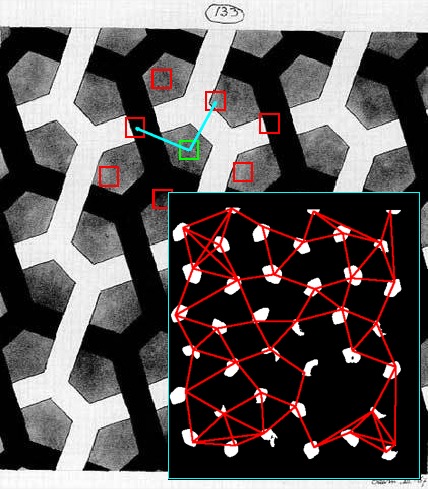}} \\    
  \subfloat[$10 \times 10$]{\includegraphics[width=.24\linewidth]{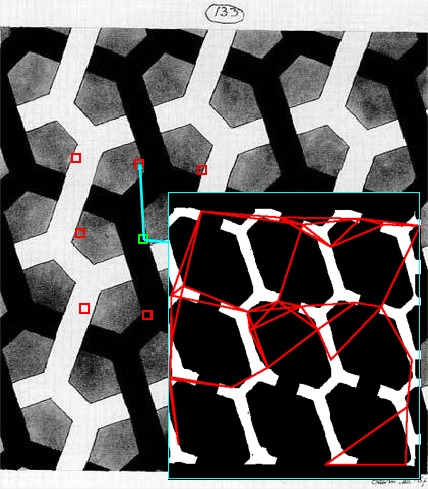}} \hspace{0.2cm}
  \subfloat[$15 \times 15$]{\includegraphics[width=.24\linewidth]{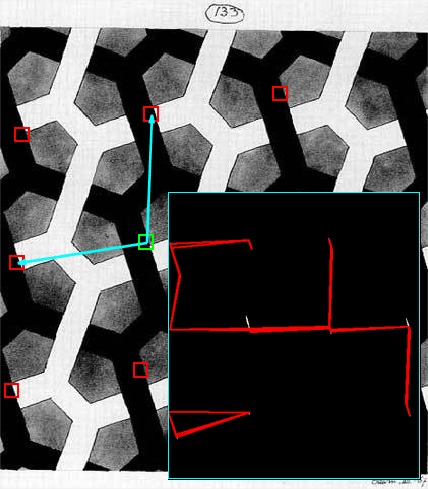}} \hspace{0.2cm}
  \subfloat[$20 \times 20$]{\includegraphics[width=.24\linewidth]{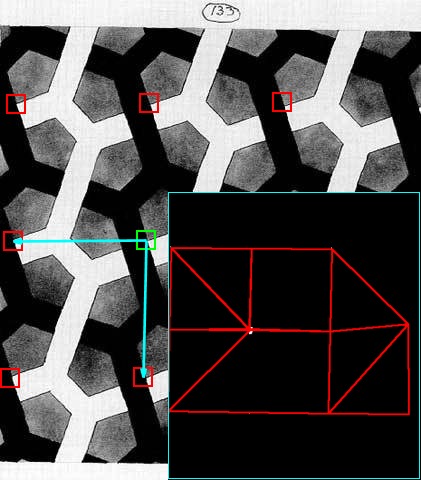}} \\
  \caption{\figuretitle{Influence of patch size and patch position} For each experiment \NFAmax \ is set to $10^4$, \ie \ 4 \% of the pixels. In most cases lower \NFAmax \ could be used but setting a high \NFAmax \ ensures that we always get detections even if the patch only contains microtexture information. Each row corresponds to a lattice proposal with same patch position but different patch sizes: $10 \times 10$ for the left column, $15 \times 15$ for the middle one and $20 \times 20$ for the right one. Each image represents the superposed proposed lattice on the original image. On the bottom-right of each image we display the underlying graph as well as the binary detection. On the first row the patch contains only a white region with a few gray pixels. The influence of these pixels is visible for small patch sizes (a) but is no longer taken into account for larger patch sizes, (b) and (c). On the second row the patch contains gray microtexture which has some local structure. We identify large similarity regions and no perceptual lattice is retrieved in (d), (e) and (f). The situation is different on the third row. The $10 \times 10$ patch contains only uniform black information in (g), but the situation changes as the patch sizes grows. In (h), the patch intersects black, gray and white zones. The graph is much sparser and the lattice is close to the perceptual one
    . In (i), the patch size is large enough to cover large areas of the three gray levels and the perceptual lattice is identified.}
  \label{fig:pos_size}
\end{figure}
\subsubsection{Crystallography images}

Defect localization, noise reduction, correction of crystalline structures in images are central tasks in crystallography. Usually, they require the knowledge of the geometry of a perfect underlying crystal. In our experiments we manually identify the geometry of the periodic crystal, which allows for multiple structures in one image, provided a user input of the primitive cell in a lattice. This primitive cell extraction could be automated \cite{mevenkamp2015unsupervised}. In Figure \ref{fig:lattices_algo}, we present an example of multiple geometry extraction. Statistics like angle and period can be retrieved using the estimated basis vectors. 
This image contains two lattices and the locality of our measurements allows for the detection of both structures. 
Using windowed Fourier transform could be efficient to obtain local measurements on the periodicity of these images since the information is highly frequential. However in order to obtain the same detection map as Algorithm \ref{alg:auto-similaritydetection} one must carefully set the threshold parameter, \NFAmax. This situation is illustrated in Figure \ref{fig:fourier_comp}.

Finally we assess the precision of our measurements by comparing our results with a model used in crystallography, see Figure \ref{fig:crystallo}. We indeed retrieve one of the possible bases used to describe these lattices. However, the symmetry constraints are not present in the identified basis. To obtain another basis, one must relax the regularization parameters. A more natural way to obtain the desired primitive cell would be to introduce symmetry constraints in the graphical model formulation in \eqref{eq:log-lik}.

\begin{figure}
  \centering
  \subfloat[]{\includegraphics[width=.24\linewidth]{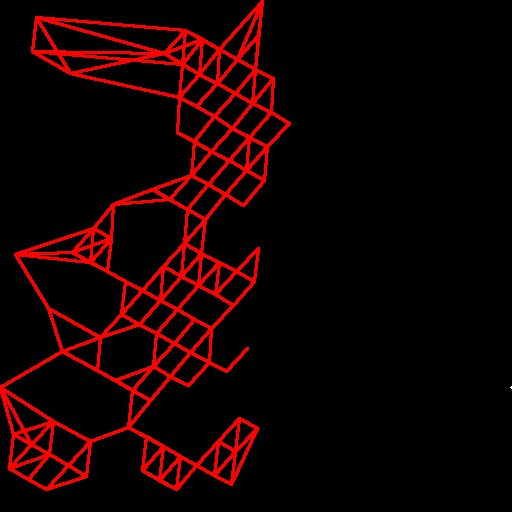}} \hfill
  \subfloat[]{\includegraphics[width=.24\linewidth]{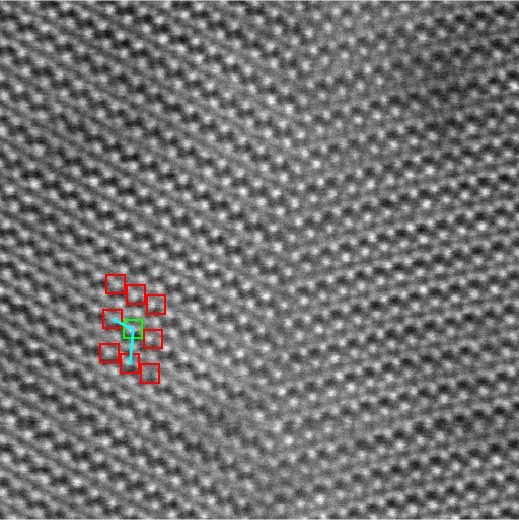}} \hfill
  \subfloat[]{\includegraphics[width=.24\linewidth]{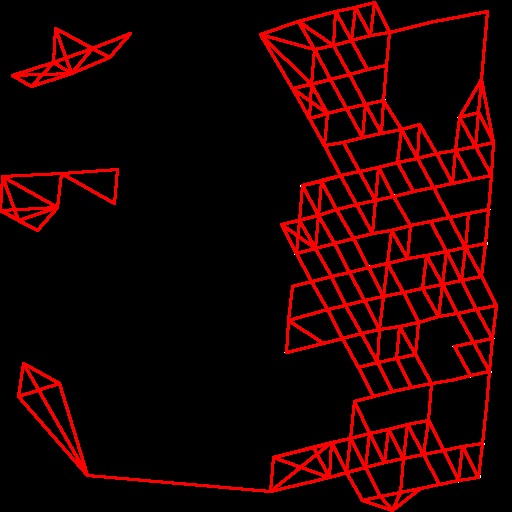}} \hfill
  \subfloat[]{\includegraphics[width=.24\linewidth]{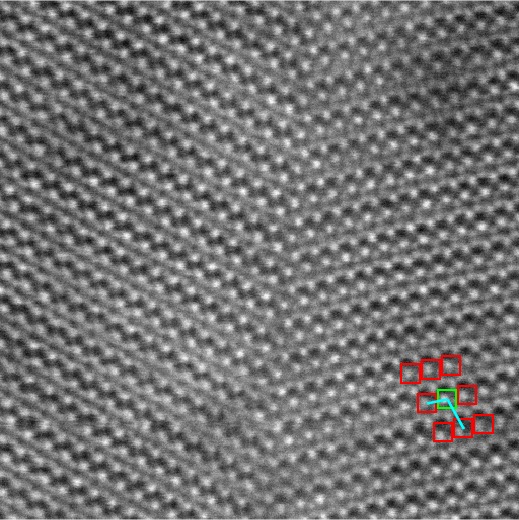}} 
  \caption{\figuretitle{Lattice extraction}In this experiment we consider a crystallographic image (an orthorombic $\textrm{NiZr}$ alloy) and set \NFAmax \ to $10^2$. Two lattices are present in this image and they are correctly identified in (b) and (d). Note that in (a), respectively in (c), mostly points in the left, respectively right, part of the image are identified, thus yielding correct lattice identification. Points which should be identified and are discarded nonetheless correspond to regions in which we observe contrast variation. Image courtesy of Denis Gratias.}
  \label{fig:lattices_algo}
\end{figure}

\begin{figure}
  \centering
  \subfloat[]{\begin{tikzpicture}[spy using outlines={rectangle, yellow,magnification=3, connect spies}]
  \node {\pgfimage[interpolate=true,height=3cm]{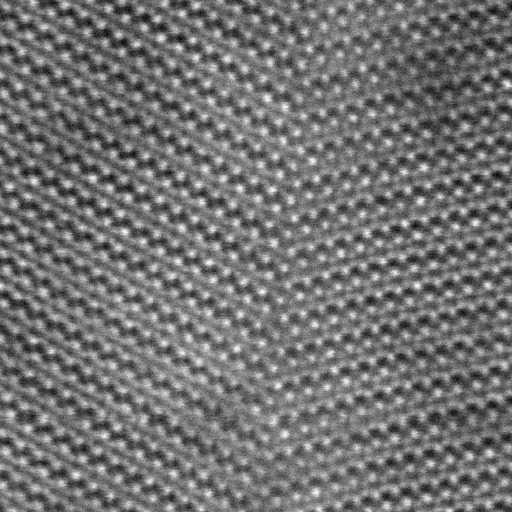}};
 \coordinate (spypoint) at (-1,1);
 \coordinate (spyviewer) at (1,-0.5);
 \spy[width=2cm,height=2cm] on (spypoint) in node [fill=white] at (spyviewer);
  
\end{tikzpicture}
} \hfill
  \subfloat[$90\%$]{\includegraphics[width=.2\linewidth]{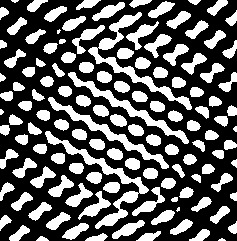}} \hfill
  \subfloat[$95\%$]{\includegraphics[width=.2\linewidth]{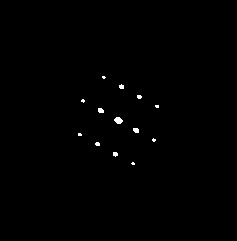}} \hfill
  \subfloat[$99\%$]{\includegraphics[width=.2\linewidth]{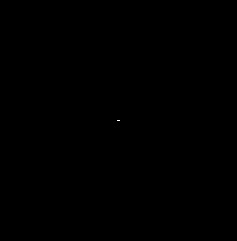}} \hfill
  \caption{\figuretitle{Comparison with Fourier based methods}Since the original image can be segmented in two highly periodic components, Fourier methods might be well-adapted to the lattice extraction task. In (a) we present a sub-image of the original alloy. We compute the autocorrelation of this sub-image and threshold it. This operation gives us a detection map, like Algorithm \ref{alg:auto-similaritydetection}. In (b) the threshold is set to $90\%$ percent of the maximum value of the autocorrelation. Too many points are identified. In (d) the threshold is set to $99\%$ and only one point is identified. Correct lattice is identified in (c). 
  }
  \label{fig:fourier_comp}
\end{figure}

\begin{figure}
  \centering
  \subfloat[]{\includegraphics[angle=90,width=.23\linewidth]{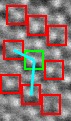}} \hspace{0.5cm}
  \subfloat[]{\includegraphics[width=.5\linewidth]{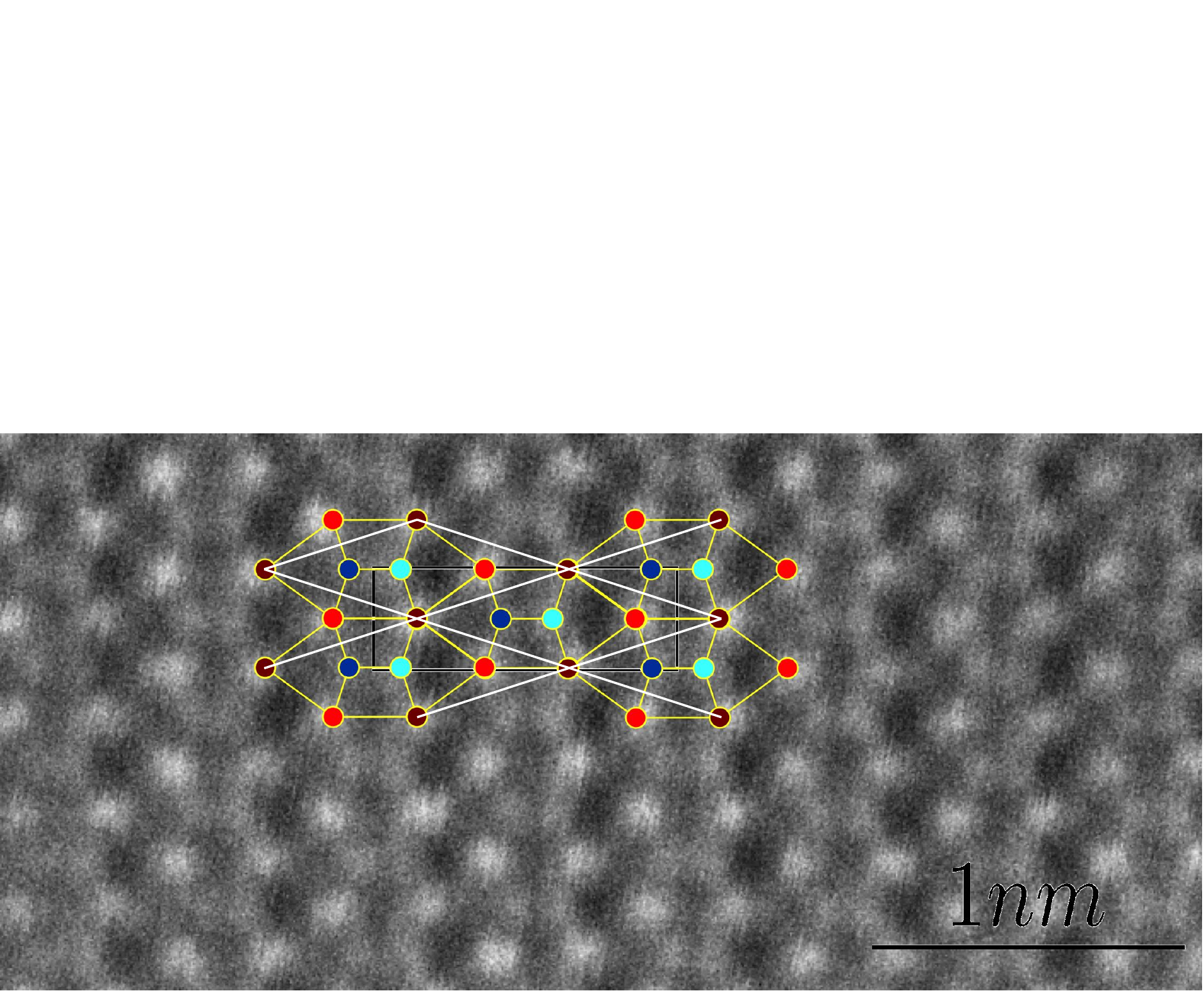}} \hfill
  \caption{\figuretitle{Agreement with crystallography models} In (a) we perform a zoom on of the lattice identified in Figure \ref{fig:lattices_algo} and compare it to the one identified by crystallographists in (b). (a) is a zoomed rotated version of a crystalline structure similar to (b). The output lattice in (a) is the same as the one in (b). Indeed in (b) the red points, for instance, form a lattice. A possible basis for this lattice is given by the vectors of a parallelogramm. Up to rotation these basis vectors match the one identified in (a). However, the parallelogramm basis is a symmetric and thus is not chosen by chemists since it does not reflect the geometry of the alloy. The preferred basis is given by the symmetric rhombus (white edges in (b)). Image courtesy of Denis Gratias.}
  \label{fig:crystallo}
\end{figure}

\subsubsection{Natural images}
\label{sec:natural_images}
Identifying lattices in natural images is a more challenging task since we have to deal with image artifacts. In this section we investigate the effect on the detection of the background clutter in natural images, see Figure \ref{fig:preprocessing}, and the effect of the camera position, see Figure \ref{fig:homography}.

\paragraph{Preprocessing}
Due to the occlusions occurring in natural images, if a lattice is superposed over a real photograph, carefully selecting structural elements might not be enough in order to retrieve the periodicity. Indeed, if we observe a repetition of the lattice pattern, the background does not necessarily contain any repetition and thus makes the detection more complicated. In order to avoid such a problem we propose to introduce a preprocessing step in our algorithm. This preprocessing step will be encoded in a linear filter $h$. Suppose $U$ is a sample from a Gaussian model with function $f$ then $h * U$ is a sample from a Gaussian model with function~$h * f$. Thus all the properties derived earlier remain valid with this linear operation. In Figure~\ref{fig:preprocessing}, we set $h$ to be a Laplacian operator \footnote{We use a discrete Laplacian operator $\Delta$ such that for any $\veclet{x} = (x_1,x_2)$, we get that $\Delta(u)(x_1,x_2) = \left(u(x_1+1,x_2) + u(x_1-1,x_2) + u(x_1,x_2+1) + u(x_1,x_2-1) - 4 u(\veclet{x})\right)/4$, where boundaries are handled periodically.}. This operation allows us to avoid contrast problems.

\begin{figure}
  \centering
  \subfloat[]{\includegraphics[width=.2\linewidth]{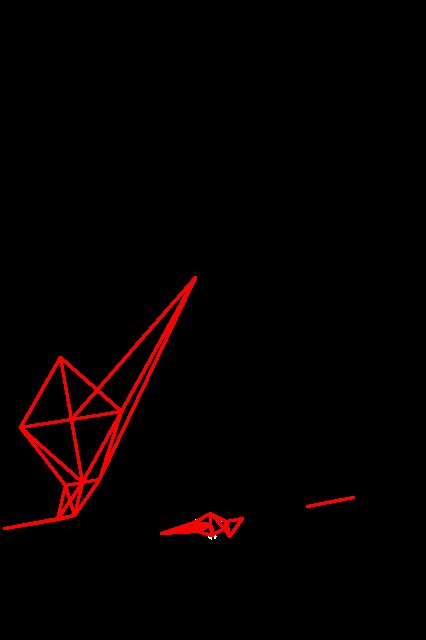}} \hfill
  \subfloat[]{\includegraphics[width=.2\linewidth]{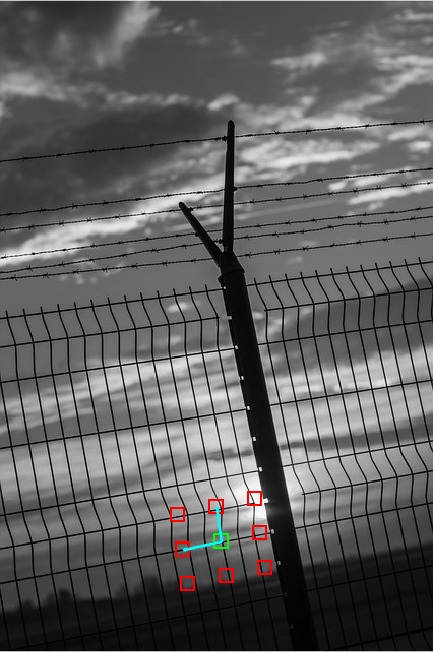}}  \hfill
  \subfloat[]{\includegraphics[width=.2\linewidth]{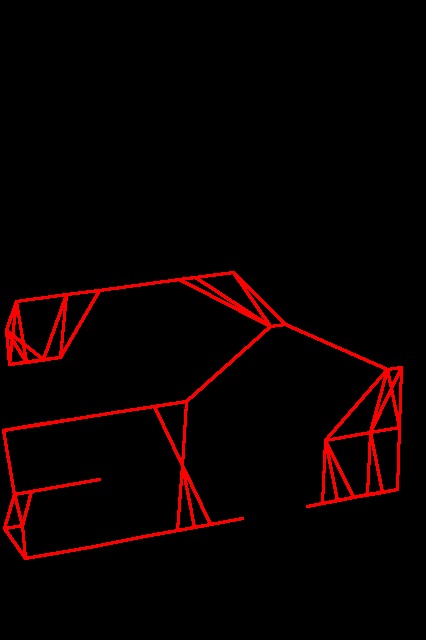}} \hfill
  \subfloat[]{\includegraphics[width=.2\linewidth]{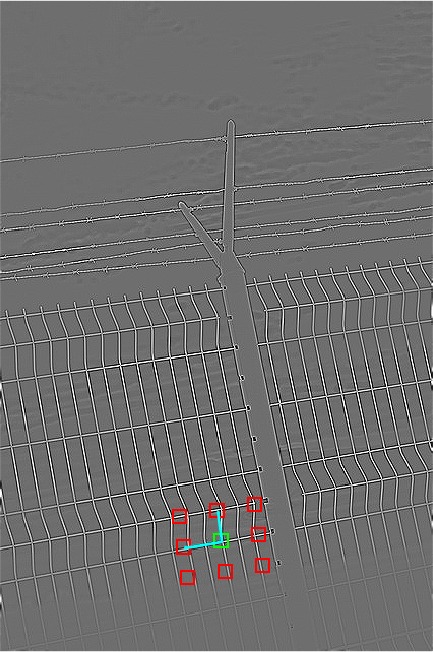}}  
  \caption{\figuretitle{Preprocessing and filtering} In (a) and (c) we display the graphs obtained with Algorithm \ref{alg:auto-similaritydetection} applied on images (b) and (d). In (b) and (d) the original image is superposed with the estimated lattice (vectors in cyan and proposed patches in red). In (a) and (b), \NFAmax \ was set to $10^5$ which corresponds to 35 \% of detection in the associated \acontrario \ model. Lower \NFAmax \ did not give enough points to conduct the lattice proposal step. We obtain a visually satisfying lattice. In (c) and (d) we apply a simple preprocessing, a Laplacian filter, to the image and set \NFAmax \ to 10. The detection figure is much cleaner and the estimation makes much more sense from a perceptual point of view. Note that, as in (b), the proposed lattice does not exactly match the fence periodicity. This is due to: 1) the initialization of the algorithm and the structure of the graph in the alternate minimization algorithm 2) the fact that the horizontal periodicity is broken by the black post.}
  \label{fig:preprocessing}
\end{figure}

\paragraph{Homography}
In the previous experiments we suppose that the lattice structure was in front of the camera. In many cases this assumption is not true and there exists an homography that matches the deformed lattice in the image to a true lattice. Our algorithm makes the assumption that the lattice is viewed in a frontal way and fails otherwise. However, locally, this assumption is true and we can observe partial match of the lattices in Figure \ref{fig:homography}.

\begin{figure}
  \centering
  \subfloat[]{\includegraphics[width=.24\linewidth]{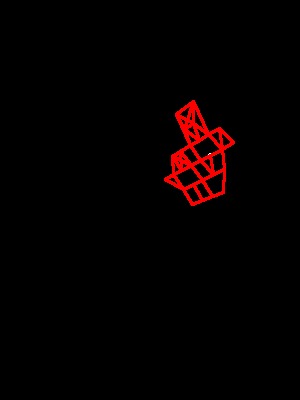}} \hspace{0.5cm}
  \subfloat[]{\includegraphics[width=.24\linewidth]{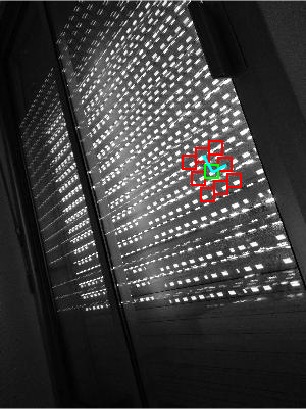}}  \\
  \caption{\figuretitle{Homography and locality} In this experiment \NFAmax \ was set to $10^3$. Note that the detected graph is localized around the original patch in (a). In (b) we superpose the proposed lattice onto the original image. The lattice proposal is valid in a small neighborhood around the original patch. However it is not valid for the whole image.}
  \label{fig:homography}
\end{figure}

\subsection{Texture ranking}
\label{sec:texture-rank} 
We conclude these experiments by showing that this simple graphical model can be used to perform ranking among texture images, sorting them according to their degree of periodicity. We say that an image has high periodicity degree if a lattice structure can be well fitted to the image. We introduce a criterion for evaluating the relevance of the lattice hypothesis. Let $u$ be an image over~$\Omega$, let $\omega \subset \Omega$ be a patch domain and $a$ be as in Proposition \ref{prop:a_contrario_bound} with \NFAmax \ set by the user.

\begin{mydef}[Periodicity criterion]
  Let $\lbrace \veclet{t} \in \Omega, \ \autosim(u,\veclet{t}, \omega) \leq a(\veclet{t}) \rbrace$ be the set of detected offsets and $N_{{\mathscr{C}}}$ its number of connected components as defined in Section \ref{sec:algorithm and properites}. Let also $(\widehat{B}, \widehat{M}, \widehat{\sigma})$ be the estimated parameters using Algorithm \ref{alg:alternate}. We define the following periodicity criterion $c_{per}$ as
  \begin{equation} c_{per}(u) = \frac{\pi \widehat{\sigma}^2}{N_{\mathscr{C}}\vertt{\operatorname{det}(\hat{b}_1,\hat{b}_2)}} \eqsp ,\label{eq:cper}\end{equation}
  where $\widehat{B} = (\hat{b}_1, \hat{b}_2)$.
\end{mydef}

The criterion $c_{per}$ simply computes the ratio between the error area of Algorithm \ref{alg:alternate}, \ie \ the error made when considering the approximated lattice hypothesis, see Definition \ref{def:approximated_lattice_hypothesis}, and the area of the parallelogram defined by the output basis vectors. If we have enough detections this quantity is supposed to be small when the approximated lattice hypothesis holds and large when it does not. Nonetheless, we introduce a dependency in the number of detections. Indeed, even if no lattice is perceived, the hypothesis in Definition \ref{def:approximated_lattice_hypothesis} may still hold if the number of detected offsets is small.

In the experiment presented in Figure \ref{fig:ranking} we sort 25 texture images based on the $c_{per}$ criterion. Images are of size ${256 \times 256}$. Since the identified graph highly depends on the patch position and the patch size, for each image we uniformly sample 150 patch positions and set the patch size to ${20 \times 20}$. For each set of parameters we find a lattice using Algorithm \ref{alg:auto-similaritydetection} and Algorithm \ref{alg:alternate} with parameters $\text{\NFAmax} = 1$, $\delta_M = 10$, $\delta_B = 10^{-2}$ and $N_{it} =10$. A statistical study of our ranking is presented in Figure \ref{fig:stat_ranking}. Note that, from a perceptual point of view, from (a) to (n) all textures are periodic except for (f), (j) and (k) which are examples for which our algorithm fails. However, from (o) to (y), no texture is periodic.

\begin{figure}
  \centering
  \subfloat[-9.75]{\includegraphics[width=.19\linewidth]{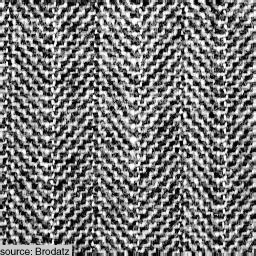}} \hfill
  \subfloat[-9.42]{\includegraphics[width=.19\linewidth]{./img/bw/img_01.jpg}} \hfill
  \subfloat[-9.12]{\includegraphics[width=.19\linewidth]{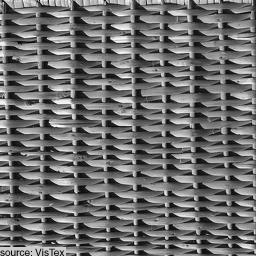}} \hfill
  \subfloat[-9.00]{\includegraphics[width=.19\linewidth]{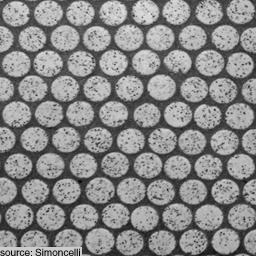}} \hfill
  \subfloat[-8.80]{\includegraphics[width=.19\linewidth]{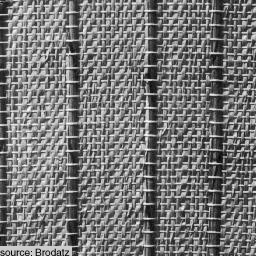}} \hfill \\
  \subfloat[-8.24]{\includegraphics[width=.19\linewidth]{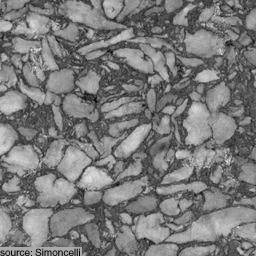}} \hfill
  \subfloat[-8.24]{\includegraphics[width=.19\linewidth]{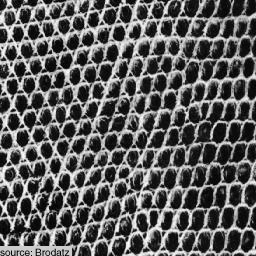}} \hfill
  \subfloat[-7.99]{\includegraphics[width=.19\linewidth]{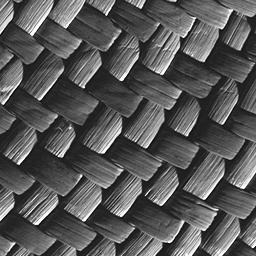}} \hfill
  \subfloat[-7.80]{\includegraphics[width=.19\linewidth]{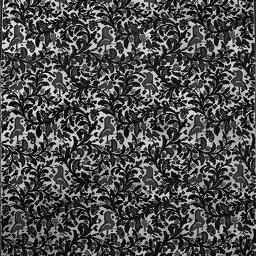}} \hfill
  \subfloat[-7.77]{\includegraphics[width=.19\linewidth]{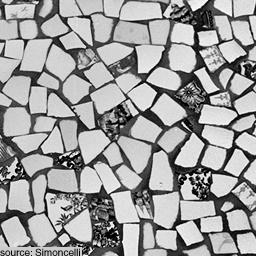}} \hfill \\
  \subfloat[-7.74]{\includegraphics[width=.19\linewidth]{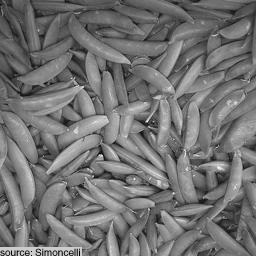}} \hfill
  \subfloat[-7.72]{\includegraphics[width=.19\linewidth]{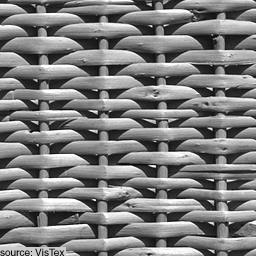}} \hfill
  \subfloat[-7.47]{\includegraphics[width=.19\linewidth]{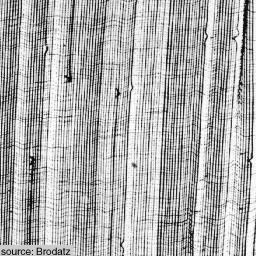}} \hfill
  \subfloat[-7.26]{\includegraphics[width=.19\linewidth]{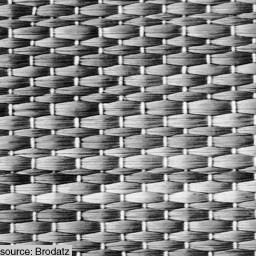}} \hfill
  \subfloat[-7.21]{\includegraphics[width=.19\linewidth]{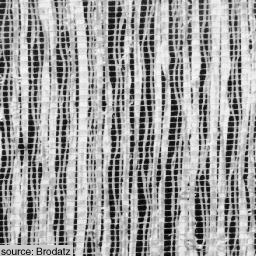}} \hfill \\
  \subfloat[-7.20]{\includegraphics[width=.19\linewidth]{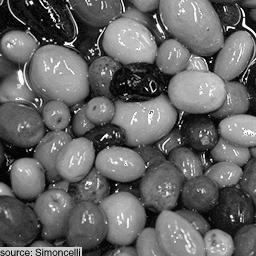}} \hfill
  \subfloat[-7.19]{\includegraphics[width=.19\linewidth]{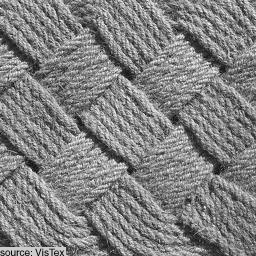}} \hfill
  \subfloat[-7.17]{\includegraphics[width=.19\linewidth]{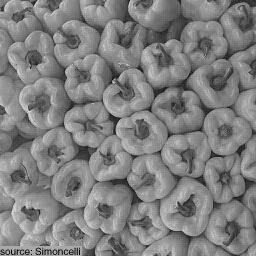}} \hfill
  \subfloat[-6.92]{\includegraphics[width=.19\linewidth]{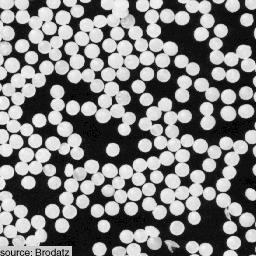}} \hfill
  \subfloat[-6.86]{\includegraphics[width=.19\linewidth]{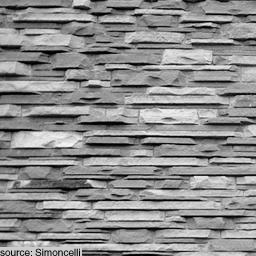}} \hfill \\
  \subfloat[-6.78]{\includegraphics[width=.19\linewidth]{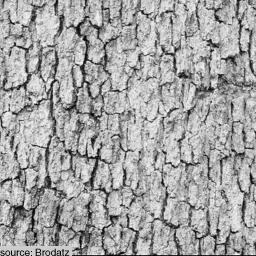}} \hfill
  \subfloat[-6.65]{\includegraphics[width=.19\linewidth]{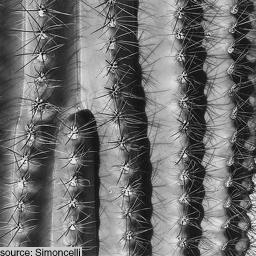}} \hfill
  \subfloat[-6.56]{\includegraphics[width=.19\linewidth]{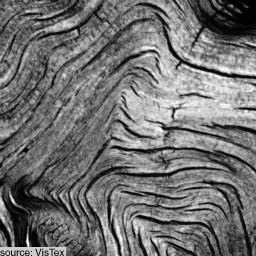}} \hfill
  \subfloat[-6.30]{\includegraphics[width=.19\linewidth]{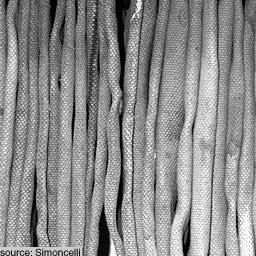}} \hfill
  \subfloat[-6.16]{\includegraphics[width=.19\linewidth]{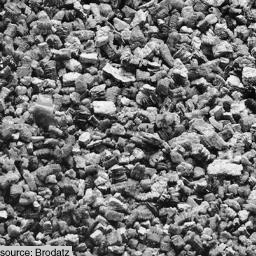}} \hfill \\  
  \caption{\figuretitle{Texture ranking}  The $c_{per}$ criterion, defined in \eqref{eq:cper}, is computed for each setting. We associate to each image the median of the 150 criterion values and sort the images accordingly. (a) corresponds to the lowest criterion, \ie \ the most periodic image according to $c_{per}$ criterion. (y) corresponds to the largest criterion, \ie \ the least periodic image according to $c_{per}$. Under each image we give the logarithm of the median $c_{per}$ values.} \label{fig:ranking}
\end{figure}

\epstopdfsetup{outdir=./img/}
\begin{figure}
  \centering
  \includegraphics[width=.5\linewidth]{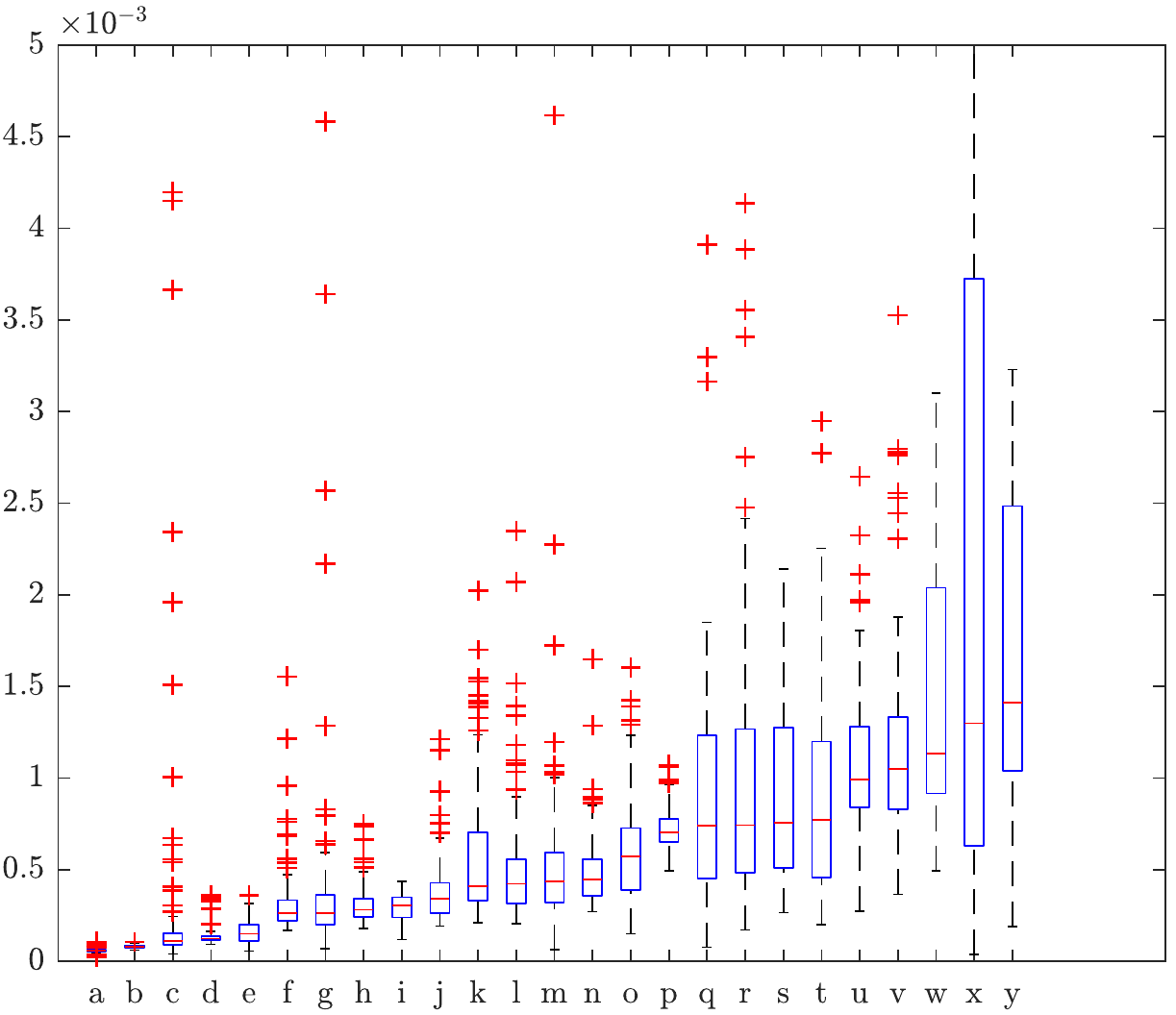}
  \caption{\figuretitle{Boxplot for $c_{per}$ values} In this figure we present a boxplot of the $c_{per}$ values, defined in \eqref{eq:cper}, used to rank textures images in Figure \ref{fig:ranking}. We recall that we use 150 random patch positions in order to compute the $c_{per}$ values. Letters on the $x$-axis correspond to the textures in Figure \ref{fig:ranking}. For each texture we present its median $c_{per}$ value. The lower, respectively upper, limit of the blue box corresponds to $25\%$, respectively $75\%$ of the computed $c_{per}$ values. The dashed line corresponds to the confidence interval with level $0.07$ under normality assumption. Points outside this interval are plotted in red and the graphics was clipped between $0$ and $5 \times 10^{-3}$. The size of the confidence interval grows with the median value. It must be noted that the overlapping of the blue boxes might explain some inconsistencies of our ranking. Another source of errors lie in the model which assumes that if a texture is periodic its pattern is described by a $20 \times 20$ patch. In order to perform a more robust ranking a multiscale approach should be preferred.} \label{fig:stat_ranking}
\end{figure}

\section{Conclusion}
\label{sec:conclusion}

In this paper we introduce a statistical model, the \acontrario \ framework, to analyze spatial redundancy in images. 
We propose a general algorithm for detecting redundancy in natural images. It relies on Gaussian random fields as background models and takes advantage of the links between the $\ell^2$ norm and Gaussian densities. The \acontrario \ formulation provides us with a statistically sound way of thresholding distances in order to assess similarity between patches. In this rationale we replace the task of manually setting thresholds by the selection of a Number of False Alarms. 

We illustrate our contribution with three examples in various domains of image processing. Introducing a simple modification of the NL-means algorithm we show that similarity detection (in this case, dissimilarity detection) in a theoretical \acontrario \ framework can easily be embedded in any image denoising pipeline. For instance, the threshold we introduced could be integrated into the Non-Local Bayes algorithm \cite{lebrun2013nonlocal} in order to estimate mean and covariance matrices with probabilistic guarantees. The generality of our model allows for several extensions for non-Gaussian noises \cite{deledalle2009iterative} or to take into account the geometry of the patch space~\cite{houdard2017high, wang2013sure}.

Turning to periodicity detection we propose a novel graphical model using the output of Algorithm \ref{alg:auto-similaritydetection} in order to extract lattices from images. In this model, lattice extraction is formulated as the maximization of some log-likelihood defined on a graph. We prove the finite-time convergence of Algorithm \ref{alg:alternate}
. We provide image experiments illustrating the role of the hyperparameters in our study and we assess the importance of selecting adaptive Gaussian random fields as background models. We remark that the expected number of false alarms parameter is linked to the choice of the input patch and give a range of possible values for \NFAmax \ settings. 
We also illustrate its possible application in crystallography as it correctly identifies underlying lattices in alloys. This rationale could be used to identify symmetry groups (wallpaper groups) in alloys, following the work of \cite{liu2004computational}. Finally our method is tested on natural images where some of its limits such as perspective defect or sensitivity to occlusion phenoma are identified. It must be noted that our method could easily be extended to color images by considering $\R^3$-valued instead of real-valued images.

Our last application consists in giving a quantitative criterion for periodicity texture ranking. This criterion is based on the parameters estimated in Algorithm \ref{alg:alternate}. Since we set our background models to be Gaussian random fields and remarking that these are good microtexture approximations we wish to explore the possibility to embed our \acontrario \ framework in texture analysis and texture synthesis algorithms. For instance an \acontrario \ methodology could be incorporated in the algorithm proposed by Raad et al. in \cite{raad2015conditional}. Another potential direction is to look at the behavior of the introduced dissimilarity functions for more general random fields in order to handle more complex and structured situations such as parametric texture synthesis.

\section{Acknowledgements}

The authors would like to thank Denis Gratias for the crystallography images, Jérémy Anger for some of natural images, Axel Davy who provided an OpenCL implementation of the NL-means algorithm and Thibaud Ehret for its insights and comments on denoising algorithms. 
\appendix
\section{Eigenvalues}
\label{sec:eigenvalues}

\begin{proof}[Proof of Proposition \ref{prop:eigenvalues}]
  \label{proof:eigenvalues}
  We fix $\veclet{t} \neq 0$ with $\| \veclet{t} \|_{\infty} < p$ and denote $C = C_{\veclet{t}}$. Without loss of generality we consider that $t_x >0$ and $t_y >0$. We consider $X$ an eigenvector of $C$. Let $\Omega_0 = \left( \Omega - \veclet{t} \right) \cap \Omega^{c}$ and the function $J: \ \Omega_0 \to \llbracket 2,+\infty \llbracket$ such that for any $\veclet{x}_0 \in \Omega_0$
  \begin{equation*}
    J(\veclet{x}_0) = \argmin \lbrace k \in  \N \without{0}, \ \veclet{x}_0 + k \veclet{t} \notin \Omega \rbrace \eqsp .
  \end{equation*}
  It is clear that $I = \lbrace (k,m), k \in \llbracket 1, m - 1 \rrbracket, \ m \in J(\Omega_0) \rbrace$ is in bijection with $\Omega$.
  Let $\veclet{x}_0 \in \Omega_0$, $m = J(\veclet{x}_0)$ and $k \in \llbracket 1,m-1 \rrbracket$. We define $X_{\veclet{x}_0, k}$ over $\Z^2$ such that
\begin{equation*}
  X_{\veclet{x}_0, k}(\veclet{x}_0 + \ell \veclet{t}) = \sin \left( \frac{\ell k \pi}{m} \right) \ \text{for} \ \ell \in \llbracket 1,m-1\rrbracket \eqsp , \qquad 0 \ \text{otherwise} \eqsp .
\end{equation*}
Using that $\sin(a+b) + \sin(a-b) = 2\sin(a) \cos(b)$,  we have for any $\veclet{x} \in \Z^2$
\begin{equation*}
  X_{\veclet{x}_0, k}(\veclet{x} + \veclet{t}) - 2 \cos \left( \frac{k\pi}{m} \right)X_{\veclet{x}_0, k}(\veclet{x}) + X_{\veclet{x}_0, k}(\veclet{x} - \veclet{t}) = 0 \eqsp .
\end{equation*}
This implies that for any $\veclet{x} \in \Z^2$ 
\begin{equation*}
  2X_{\veclet{x}_0, k}(\veclet{x}) - X_{\veclet{x}_0, k}(\veclet{x} + \veclet{t}) - X_{\veclet{x}_0, k}(\veclet{x} - \veclet{t}) = \left[2 - 2\cos\left( \frac{k\pi}{m} \right)\right] X_{\veclet{x}_0, k}(\veclet{x}) = 4\sin^2\left( \frac{k \pi}{m}\right)X_{\veclet{x}_0, k}(\veclet{x}) \eqsp .
\end{equation*}
Thus the one-dimensional vector (given by the raster-scan order on the $x$-axis) of the restriction of $X_{\veclet{x}_0, k}$ is an eigenvector of $C$ associated with eigenvalue $4\sin^2\left( \frac{k \pi}{m}\right)$.

Using that $I$ is in bijection with $\Omega$ we get that the number of vectors $(X_{\veclet{x}_0, k})$ is $|\Omega|$. We show that this family of vectors is linearly-independent. Let $\Lambda_{\veclet{x}_0, k} \in \R$ such that
\begin{equation*}
  \sum_{\veclet{x}_0 \in \Omega_0} \sum_{k=1}^{J(\veclet{x}_0) - 1} \Lambda_{\veclet{x}_0, k}X_{\veclet{x}_0,k} = 0 \eqsp .
\end{equation*}
Since $X_{\veclet{x}_0,k}$ and $X_{\veclet{y}_0,k'}$ have different support if and only if $\veclet{x}_0 \neq \veclet{y}_0$ we get that for any $\veclet{x}_0 \in \Omega_0$, $\sum_{k=1}^{J(\veclet{x}_0) -1} \Lambda_{\veclet{x}_0,k} X_{\veclet{x}_0,k} = 0$. This gives that $(\Lambda_{\veclet{x}_0,k})_{k \in \llbracket 1, J(\veclet{x}_0) - 1 \rrbracket}$ is in the kernel of the matrix $\left( \sin(\ell k \pi / (J(\veclet{x}_0) - 1)) \right)_{1 \leq j, \ell \leq J(\veclet{x}_0) - 1}$. Since the sinus discrete transform is invertible we obtain that for any $\veclet{x}_0 \in \Omega_0$ and $k \in \llbracket 1, J(\veclet{x}_0) - 1 \rrbracket$, $\Lambda_{\veclet{x}_0,k} = 0$. Thus the family $X_{\veclet{x}_0, k}$ is a basis of eigenvectors.

We aim at computing the cardinality of $K_{k,m} = \lbrace X_{\veclet{x}_0,k}, J(\veclet{x}_0) = m \rbrace$. By definition, in Proposition \ref{prop:eigenvalues}, $r_{k,m} = |K_{k,m}|$. First note that $|K_{k',m}| = |K_{k,m}|$. We give the following decomposition $\Omega_0 = \Omega_x \cup \Omega_y \cup \Omega_{x,y}$ with
\begin{equation*}
  \Omega_x = \llbracket -t_x, -1 \rrbracket \times \llbracket 0, p - 1 - t_y \rrbracket, \ \Omega_x = \llbracket 0, p-1-t_x  \rrbracket \times \llbracket -t_y, -1 \rrbracket, \ \Omega_{x,y} = \llbracket -t_x, -1 \rrbracket \times \llbracket -t_y, -1 \rrbracket \eqsp .
\end{equation*}
Note that for all $\veclet{x}_0 \in \Omega_0$ we have that $\veclet{x}_0 + (q+1)\veclet{t} \notin \Omega$, with $q = \lceil \frac{p}{|t_x| \vee |t_y|} \rceil $. Thus $J(\Omega_0) \subset \llbracket 2, q+1 \rrbracket$.  Let $m \in \llbracket 2, q-1 \rrbracket$. The cardinality of $K_{k,m}$ is the cardinality of $J^{-1}(m)$. Let $\veclet{x}_0 \in \Omega_x$ we have
\begin{equation*}
  \veclet{x}_0=(i_0, j_0)  \in K_{k,m} \ \Leftrightarrow \ \accolade{& i_0 + m t_x \geq p \ \\ &\text{or} \\ & j_0 + mt_y \geq p} \text{and} \accolade{& i_0 + (m-1) t_x \leq p-1 \ \\ &\text{and} \\ &j_0 + (m-1) t_y \leq p-1 \\} \eqsp .
\end{equation*}
Since $\veclet{x}_0 \in \Omega_x$ we have $i_0 + mtx \leq p-1$, hence
\begin{equation*}
  \veclet{x}_0=(i_0, j_0)  \in K_{k,m} \ \Leftrightarrow \  j_0 + mt_y \geq p \ \text{and} \ j_0 + (m-1) t_y \leq p-1 \eqsp .
\end{equation*}
Thus $| \Omega_x \cap J^{-1}(m) | = t_x t_y$. Similarly we get that $| \Omega_y \cap J^{-1}(m) | = t_x t_y$ and $\Omega_{x,y} \cap J^{-1}(m) = \emptyset$. Thus, $|K_{k,m}| = 2 t_x t_y$.

We have computed $|K_{k,m}|$ for every $m \in \llbracket 2, q-1 \rrbracket$. In order to complete our study it only remains to compute $|K_{k,q+1}|$, since $|K_{k,q}|$ can be deduced from the summability condition and from $|K_{k,m}| = |K_{k',m}|$. We only compute $|K_{k, q+1}|$. We remark that $\Omega_x \cap J^{-1}(q+1) = \Omega_y \cap J^{-1}(q+1) = \emptyset$. Let $\veclet{x}_0 \in \Omega_{x,y}$ then $\veclet{x}_0 = -\veclet{t} + (x,y)$ with $x \in \llbracket 0, t_x - 1 \rrbracket$ and $y \in \llbracket 0, t_y - 1 \rrbracket$. We obtain the following equivalence
\begin{equation*}
  \veclet{x}_0 \in J^{-1}(q+1) \ \Leftrightarrow \ \accolade{&-t_x + x + (q+1)t_x \geq p \ \\ &\text{or} \\ &-t_y + y + (q+1)t_y  \geq p} \text{and} \accolade{&-t_x + x + qt_x  \leq p-1 \ \\ & \text{and} \\ & -t_y + y + qt_y \leq p-1 \\} \eqsp .
\end{equation*}
Since $qt_x \geq p$ or $qt_y \geq p$ we obtain that the first condition is always satisfied. Thus we get
\begin{equation*}
  \veclet{x}_0 \in J^{-1}(q+1) \ \Leftrightarrow \ x \leq p - 1 - (q-1)t_x \ \text{and} \ y \leq p - 1 - (q-1)t_y \eqsp .
\end{equation*}
Using that $p-1-(q-1)t_x = \left( \lceil \frac{p}{t_x} \rceil - q \right) t_x + t_x -1 - p_x$, we conclude the proof.
\end{proof}
  
\section{Update rules}
We derive the proof of Proposition \ref{prop:alternate_update}.
  \begin{proof}
    \label{proof:alternate_update}
    Computing the minimum of $q(B,M | E)$ for fixed $B \in \R^4$, respectively fixed $M \in \R^{2|E|}$, gives the update rule for $M$, respectively for $B$. We obtain that 
    \[ \al{
        q(B,M | E) &= \summ{\veclet{e} \in E}{}{m_{\veclet{e}}^2\| b_1 \|^2} + \summ{\veclet{e} \in Eb}{}{n_{\veclet{e}}^2 \| b_2 \|^2} + 2 \summ{e \in E}{}{m_{\veclet{e}} n_{\veclet{e}} \langle b_1, b_2 \rangle} \\
        &- 2 \summ{e \in E}{}{m_{\veclet{e}} \langle b_1, \veclet{e} \rangle} - 2 \summ{e \in E}{}{n_{\veclet{e}} \langle b_2, \veclet{e} \rangle} + r(B,M)\\
        &=  B^T \left( \Lambda_M \otimes \operatorname{Id}_2 \right) B - 2 \langle B, E_M \rangle + \alpha(M)\\
        &= \| \left(\Lambda_M \otimes \operatorname{Id}_2 \right)^{\frac{1}{2}}  B - \left(\Lambda_M \otimes \operatorname{Id}_2 \right)^{\frac{-1}{2}}E_M \|^2 + \alpha(M) \\
        &= \| \left(\Lambda_M \otimes \operatorname{Id}_2 \right)^{\frac{1}{2}}  \left(B - \left(\Lambda_M \otimes \operatorname{Id}_2 \right)^{-1}E_M \right) \|^2 + \alpha(M) \eqsp ,} \]
    where $\alpha(M)$ depends only on $M$. Similar derivation goes for $B$ and we obtain the proposed update rules.
  \end{proof}


\bibliographystyle{plain}
\bibliography{research.bib}

\end{document}